\documentclass[11pt,oneside]{article}

\usepackage{config_mjh_paper_arxiv}
\usepackage{config_mjh_symbols_arxiv}

\usepackage[numbers]{natbib}
\usepackage{amsmath}                                          %
\usepackage{amsthm}                                           %
\theoremstyle{definition} \newtheorem{defn}{Definition}       
\theoremstyle{plain}            
\theoremstyle{plain} \newtheorem{thm}[defn]{Theorem}                
\theoremstyle{plain} \newtheorem{lem}[defn]{Lemma}                  
\theoremstyle{plain}               
\theoremstyle{remark} \newtheorem{rmk}[defn]{Remark}                
\theoremstyle{remark}                 

\usepackage{enumitem} 
\makeatletter
\def\namedlabel#1#2{\begingroup
    #2%
    \def\@currentlabel{#2}%
    \phantomsection\label{#1}\endgroup
}
\makeatother

\usepackage[pdfauthor={Matthew James Holland},%
colorlinks=true,%
linkcolor=blue,%
citecolor=blue]{hyperref}
\hypersetup{pdftitle={Holland and Haress (2020): Learning with CVaR-based feedback under potentially heavy tails}} %

\begin{document}

\title{\textbf{Learning with CVaR-based feedback\\under potentially heavy tails}}
\author{
  Matthew J.~Holland\thanks{Please direct correspondence to \texttt{matthew-h@ar.sanken.osaka-u.ac.jp}.}\\
  Osaka University
  \and
  El Mehdi Haress\\
  CentraleSup\'{e}lec
}
\date{} 

\maketitle

\begin{abstract}
We study learning algorithms that seek to minimize the conditional value-at-risk (CVaR), when all the learner knows is that the losses incurred may be heavy-tailed. We begin by studying a general-purpose estimator of CVaR for potentially heavy-tailed random variables, which is easy to implement in practice, and requires nothing more than finite variance and a distribution function that does not change too fast or slow around just the quantile of interest. With this estimator in hand, we then derive a new learning algorithm which robustly chooses among candidates produced by stochastic gradient-driven sub-processes. For this procedure we provide high-probability excess CVaR bounds, and to complement the theory we conduct empirical tests of the underlying CVaR estimator and the learning algorithm derived from it.
\end{abstract}

\section{Introduction}\label{sec:intro}

In machine learning problems, since we only have access to limited information about the underlying data-generating phenomena or goal of interest, there is significant uncertainty inherent in the learning task. As a result, any meaningful performance guarantee for a learning procedure can only be stated with some degree of confidence (e.g., a high probability ``good performance'' event), usually with respect to the random draw of the data used for training. Assuming some loss $\loss(w;z) \geq 0$ depending on parameter $w \in \WW \subseteq \RR^{d}$ and data realization $z \in \ZZ$, given random data distributed as $Z \sim \ddist$, the \textit{de facto} standard performance metric in machine learning is the \emph{risk}, or expected loss, defined
\begin{align}\label{eqn:risk_defn}
\risk(w) \defeq \exx_{\ddist}\loss(w;Z) = \int_{\ZZ} \loss(w;z) \, \ddist(\dif z), \qquad w \in \WW.
\end{align}
The vast majority of research done on machine learning algorithms provides performance guarantees stated in terms of the risk \citep{haussler1992a,devroye1996ProbPR,anthony1999NNTheory}. This risk-centric paradigm goes beyond the theory and reaches into the typical workflow of any machine learning practitioner, since ``off-sample performance'' is typically evaluated by using the average loss on a separate set of ``test data,'' an empirical counterpart to the risk studied in theory. While the risk is convenient in terms of probabilistic analysis, it is merely one of countless possible descriptors of the distribution of $\loss(w;Z)$. When using a learning algorithm designed to minimize the risk, one makes an implicit value judgement about how the learner should be penalized for ``typical'' mistakes versus ``atypical'' but egregious errors.

As machine learning techniques are applied in increasingly diverse domains, it is important to make this value judgement more explicit, and to offer users more flexibility in controlling the ultimate \emph{goal} of learning. One of the best-known alternatives to the risk is the \emph{conditional value-at-risk} (CVaR), which considers the expected loss, conditioned on the event that the loss exceeds a user-specified $(1-\alpha)$-level quantile, here denoted for each $w \in \WW$ as
\begin{align}
\label{eqn:cc_defn}
\cc_{\alpha}(w) \defeq \frac{1}{\alpha} \exx_{\ddist} \loss(w;Z) I_{\{\loss(w;Z) \geq \vv_{\alpha}(w)\}} = \frac{1}{\alpha} \int_{\loss(w;z) \geq \vv_{\alpha}(w)} \loss(w;z) \, \ddist(\dif z),
\end{align}
where $\vv_{\alpha}(w) \defeq \inf \left\{u \in \RR: \ddist\{ \loss(w;Z) \leq u \} \geq 1-\alpha \right\}$ (called \emph{value-at-risk}, or VaR). Driven by influential work by \citet{artzner1999a} and \citet{rockafellar2000a}, under known parametric models, the problem of estimating and minimizing the CVaR reliably and efficiently has been rigorously studied, leading to a wide range of applications in finance \citep{krokhmal2002a,mansini2007a}, and even some specialized settings of machine learning tasks \citep{takeda2008a,chow2016a}. In general machine learning tasks, however, a non-parametric scenario is more typical, where virtually nothing is known about the distribution of $\loss(w;Z)$, adding significant challenges to both the design and analysis of procedures designed to minimize the CVaR with high confidence.

\paragraph{Our contributions}

In this work, we consider the case of potentially heavy-tailed losses, namely a learning setup in which all the learner knows is that the distribution of $\loss(w;Z)$ has finite variance. It is unknown in advance whether the losses are statistically congenial in the sub-Gaussian sense, or highly susceptible to outliers with infinite higher-order moments. Our main contributions:
\begin{itemize}
\item New error bounds for a large class of estimators of the CVaR for potentially heavy-tailed random variables (Algorithm \ref{algo:static}, Theorem \ref{thm:error_bd_static}).

\item A general-purpose learning algorithm which runs stochastic GD sub-processes in parallel and uses the new CVaR estimators to robustly validate the strongest candidate (Algorithm \ref{algo:dynamic}), which enjoys sharp excess CVaR bounds (Theorem \ref{thm:error_bd_dynamic}).

\item An empirical study (section \ref{sec:empirical}) highlighting the potential computational advantages and robustness of the proposed approach to CVaR-based learning.
\end{itemize}

\paragraph{Review of related work}

To put the contributions stated above in context, we give an overview of the two key strands of technical literature that are closely related to our work. First, an interesting line of work has recently developed which handles risk-averse learning scenarios where the losses can be heavy-tailed, with key works due to \citet{kolla2019a}, \citet{prashanth2019a}, \citet{bhat2020a}, and \citet{kagrecha2020a}. These works all consider some kind of sub-routine for robustly estimating the CVaR, as we do as well. The actual estimation procedures and proof techniques differ, and we provide a detailed comparison of resulting error bounds in section \ref{sec:bound_compare}. Furthermore, the latter three works only consider rather specialized learning algorithms in the context of bandit-like online learning problems, whereas the generic gradient-based procedures we study in section \ref{sec:theory_dynamic} have a much wider range of applications. Second, recent work from \citet{cardoso2019a} and \citet{soma2020a} also consider tackling the CVaR-based learning problem using general-purpose gradient-based stochastic learning algorithms. However, these works assume a bounded (and thus sub-Gaussian) loss; we discuss differences in technical assumptions in detail in Remark \ref{rmk:compare_algos}, but the most important difference is that their setup precludes the possibility of heavy-tailed losses and is thus more restrictive statistically than ours, which naturally leads to different algorithms, proof techniques, and performance guarantees.

\section{Theoretical analysis}

This section is broken into three sub-sections. First we establish notation and basic technical conditions in section \ref{sec:theory_setup}. We then study pointwise CVaR estimators in section \ref{sec:theory_static}, and subsequently leverage these results to derive a new learning algorithm with performance guarantees in section \ref{sec:theory_dynamic}.

\subsection{Preliminaries}\label{sec:theory_setup}

In the context of learning problems, random variable $Z$ denotes our data, taking values in some measurable space $\ZZ$ with $\ddist$ the probability measure induced by $Z$. The set $\WW \subseteq \RR^{d}$ is a parameter set from which the learning algorithm chooses an element. We reinforce the point that the ultimate formal goal of learning here is to minimize $\cc_{\alpha}(\cdot)$ defined in (\ref{eqn:cc_defn}) over $\WW$, where $0 < \alpha < 1$ is a user-specified risk-level parameter. This is in contrast with the traditional risk-centric setup, which seeks to minimize $R(\cdot)$ defined in (\ref{eqn:risk_defn}). For the pointwise estimation problem in section \ref{sec:theory_static} to follow, to cut down on excess notation, we simply take $X = \loss(w;Z)$, re-christen $\ddist$ as the distribution of $X$, and write the distribution function as $F_{\ddist}(u) \defeq \ddist\{X \leq u\}$ for $u \in \RR$. Similarly, since the choice of $w \in \WW$ is not important in section \ref{sec:theory_static}, there we shall write simply $\cc_{\alpha}$ and $\vv_{\alpha}$ for the CVaR and VaR of $X$, and return to the $w$-dependent notation $\cc_{\alpha}(w)$ and $\vv_{\alpha}(w)$ in section \ref{sec:theory_dynamic}. For any $m \geq 1$, we denote by $[m] \defeq \{1,\ldots,\lfloor m \rfloor\}$ all positive integers less than or equal to $m$. Finally, let $I_{\{\texttt{event}\}}$ denote the indicator function, returning $1$ when $\texttt{event}$ is true, and $0$ otherwise.

Regarding technical assumptions, we shall henceforth assume that $F_{\ddist}: \RR \to [0,1]$ is continuous, which in particular implies that $F_{\ddist}(\vv_{\alpha}) = \ddist\{ X \leq \vv_{\alpha} \} = 1-\alpha$ for all $\alpha$. This setup is entirely traditional; see for example the well-known work of \citet{rockafellar2000a}. In general, if $F_{\ddist}$ has flat regions, there may be infinitely many $1-\alpha$ quantiles; here $\vv_{\alpha}$ as introduced in section \ref{sec:intro} is simply defined to be the smallest one. See Figure \ref{fig:schematic_vv} for an illustration. The key technical assumption that will be utilized is as follows:
\begin{itemize}
\item[\namedlabel{asmp:cdf_growth}{A1}.] There exists values $0 < \parasuf < \parasm < \infty$ such that for any $|u| \leq 1$, the distribution function induced by $\ddist$ satisfies $\parasuf u \leq |F_{\ddist}(\vv_{\alpha}+u)-F_{\ddist}(\vv_{\alpha})| \leq \parasm u$.
\end{itemize}
Obviously, we are assuming that $\vv_{\alpha} \pm 1$ are within the domain of $X \sim \ddist$; this is only for notational simplicity, and the range can be taken arbitrarily small. In words, assumption $\text{\ref{asmp:cdf_growth}}(\parasuf,\parasm)$ is a \emph{local} assumption of both a $\parasm$-Lipschitz property and a $\parasuf$-growth property, local in the sense that it need only hold around the particular point $\vv_{\alpha}$ of interest. The former property ensures that $F_{\ddist}$ cannot jump with arbitrary steepness in the region of interest. The latter ensures that $F_{\ddist}$ is not flat in this region. Finally, we remark that the property of $\parasuf$-growth is utilized in key recent work done on concentration of CVaR estimators under potentially heavy-tailed data, including \citet[Prop.~2]{kolla2019a} and \citet[Lem.~5.1]{prashanth2019a}.

\begin{figure}[t]
\centering
\includegraphics[width=0.75\textwidth]{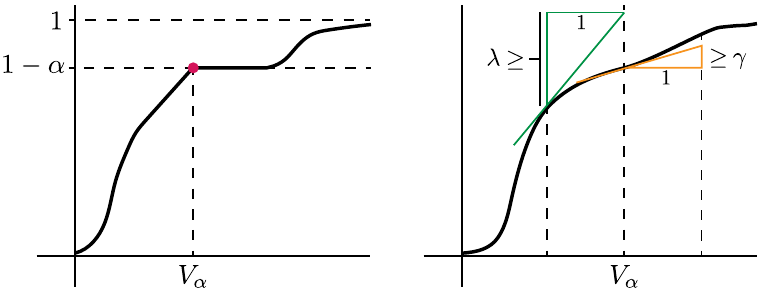}
\caption{A simple schematic illustrating $\vv_{\alpha}$ and the condition $\text{\ref{asmp:cdf_growth}}(\parasuf,\parasm)$.}
\label{fig:schematic_vv}
\end{figure}

\subsection{Robust estimation of the CVaR criterion}\label{sec:theory_static}

We begin by considering pointwise estimates, assuming that $X \sim \ddist$ is a non-negative random variable, and that we have $2n$ independent copies of $X$, denoted $\X_{n} \defeq \{X_{1},\ldots,X_{n}\}$ for the first half, and $\Y_{n} \defeq \{Y_{1},\ldots,Y_{n}\}$ for the second half. The latter half will be used to construct an estimator $\vvhat_{\alpha} \approx \vv_{\alpha}$. The former half, with $\vvhat_{\alpha}$ in hand, will be used to construct an estimator $\cchat_{\alpha} \approx \cc_{\alpha}$. As an initial approach to the problem, note that we can decompose the deviations as
\begin{align}
\nonumber
\left|\cchat_{\alpha} - \cc_{\alpha}\right| & = \frac{1}{\alpha} \left| \alpha\,\cchat_{\alpha} - \exx_{\ddist} X \, I_{\{X \geq \vvhat_{\alpha}\}} + \exx_{\ddist} X \, I_{\{X \geq \vvhat_{\alpha}\}} - \exx_{\ddist} X \, I_{\{X \geq \vv_{\alpha}\}} \right|\\
\label{eqn:decomposition}
& \leq \frac{1}{\alpha} \left( \left|\alpha\,\cchat_{\alpha} - \exx_{\ddist} X \, I_{\{X \geq \vvhat_{\alpha}\}}\right| + \left|\exx_{\ddist} X\left(I_{\{X \geq \vvhat_{\alpha}\}} - I_{\{X \geq \vv_{\alpha}\}}\right) \right| \right).
\end{align}
This gives us two terms to control. Starting with the left-most term, let us first make the notation a bit easier to manage. Conditioning on $\Y_{n}$ makes $\vvhat_{\alpha} \in \RR$ a fixed value, and based on this, we define
\begin{align}
X^{\prime} \defeq X \, I_{\{ X \geq \vvhat_{\alpha} \}}.
\end{align}
Since $\vvhat_{\alpha}$ is computed based on available data, and $X$ is observable, it follows that $X^{\prime}$ itself is observable. Denote the corresponding sample by $\X^{\prime}_{n} \defeq \{X_{1}^{\prime},\ldots,X_{n}^{\prime}\}$, where we set $X_{i}^{\prime} \defeq X_{i} \, I_{\{ X_{i} \geq \vvhat_{\alpha} \}}$. The most direct approach to this problem is to simply pass this transformed dataset $\X^{\prime}_{n}$ to a sufficiently robust sub-routine for mean estimation. More precisely, we desire a sub-routine $\rmean$ by which assuming only $\exx_{\ddist}X^{2} < \infty$, for any choice of $\delta \in (0,1)$, we can guarantee
\begin{align}\label{eqn:subg_estimator}
\prr\left\{ \left|\rmean\left[\X_{n}^{\prime}\right]-\exx_{\ddist}X^{\prime}\right| > c \, \sigma^{\prime} \sqrt{\frac{1+\log(\delta^{-1})}{n}} \right\} \leq \delta,
\end{align}
where $c>0$ is a constant depending only on the nature of $\rmean$, $\sigma^{\prime}$ is any quantity bounded as $\sigma^{\prime} \leq \sqrt{\exx_{\ddist}(X^{\prime})^{2}}$, and probability is taken with respect to the random draw of $\X_{n}$. The final estimator of interest, then, using $2n$ observations in total, will simply be defined as
\begin{align}\label{eqn:cchat_defn}
\cchat_{\alpha} \defeq \frac{1}{\alpha} \cchat_{\alpha}^{\prime}\left[\X_{n},\Y_{n}\right], \text{ where } \cchat_{\alpha}^{\prime}\left[\X_{n},\Y_{n}\right] \defeq \rmean\left[\X_{n}^{\prime}\right].
\end{align}
This general procedure is summarized in Algorithm \ref{algo:static}.
\begin{algorithm}[t!]
\caption{Scaled CVaR under potentially heavy-tailed data; $\displaystyle \cchat_{\alpha}^{\prime}\left[\X_{n},\Y_{n}\right]$.}
\label{algo:static}
\begin{algorithmic}
\State \textbf{inputs:} samples $\X_{n}$ and $\Y_{n}$, risk level $\alpha \in (0,1)$, robust sub-routine $\rmean$.
\medskip
\State Sort ancillary data $\displaystyle Y_{1}^{\ast} \leq Y_{2}^{\ast} \leq \ldots \leq Y_{n}^{\ast}$.
\medskip
\State Set threshold $\displaystyle \vvhat_{\alpha} = Y_{\lfloor(1-\alpha)n\rfloor}^{\ast}$.
\medskip
\State Augment data $\displaystyle X_{i}^{\prime} = X_{i} \, I_{\{ X_{i} \geq \vvhat_{\alpha} \}}$, for $i \in [n]$.
\medskip
\State \textbf{return:} $\displaystyle \cchat_{\alpha}^{\prime}\left[\X_{n},\Y_{n}\right] = \rmean\left[\{X_{i}^{\prime}: i \in [n]\}\right]$.
\end{algorithmic}
\end{algorithm}

Before proceeding any further, the first question to answer is whether or not such a procedure $\rmean$ can be constructed. Fortunately, since $\X_{n}$ and $\Y_{n}$ are independent, there are computationally efficient procedures which satisfy the key requirement (\ref{eqn:subg_estimator}). For concreteness, some well-known and useful examples of $\widehat{u} = \rmean[\{u_{1},\ldots,u_{n}\}]$ for arbitrary real values $u_{i}$ are as follows:
\begin{align}
\label{eqn:rmean_mom}
\widehat{u}_{\texttt{MoM}} & = \med\{ \overbar{u}^{(1)},\ldots,\overbar{u}^{(k)}\}\\
\label{eqn:rmean_cat}
\widehat{u}_{\texttt{Cat}} & = \argmin_{v \in \RR} \sum_{i=1}^{n} \rho\left(\frac{u_{i}-v}{s}\right)\\
\label{eqn:rmean_lm}
\widehat{u}_{\texttt{LM}} & = \frac{1}{n} \sum_{i=1}^{n} u_{i} \, I_{\{ a \leq u_{i} \leq b \}}\\
\label{eqn:rmean_hol}
\widehat{u}_{\texttt{Hol}} & = \frac{s}{n} \sum_{i=1}^{n} \psi\left(\frac{u_{i}}{s}\right).
\end{align}
The subscript $\texttt{MoM}$ refers to classical median-of-means, and thus the set of $n$ points is partitioned into $k$ disjoint subsets, with $\overbar{u}^{(j)}$ referring to the arithmetic mean computed on the $j$th subset \citep{lerasle2011a,hsu2016a}. The estimator marked $\texttt{Cat}$ refers to any M-estimator such that the convex function $\rho$ is differentiable, and $\rho^{\prime}$ satisfies the key conditions put forward by \citet{catoni2012a}, with $s>0$ being a scaling parameter. The estimator marked $\texttt{LM}$ refers to the truncated mean estimator studied by \citet[Sec.~2]{lugosi2019a}, where $a$ and $b$ are set using quantiles and a sample-splitting procedure. Finally, the estimator marked $\texttt{Hol}$ is the soft truncation estimator studied by \citet[Sec.~3]{holland2020a}, where $s>0$ is a scaling parameter and $\psi$ is a particular sigmoid function. In the following lemma, we summarize the robust mean estimation performance guarantees available for these estimators.
\begin{lem}[Procedures for good $\X_{n}$ event]\label{lem:robust_mean_subroutine}
The implementations of $\rmean$ given in equations (\ref{eqn:rmean_mom})--(\ref{eqn:rmean_hol}) satisfy (\ref{eqn:subg_estimator}) at confidence level $\delta$, as follows.
\begin{itemize}
\item $\textup{\texttt{MoM}:}$ with $c \leq 2\sqrt{e}$ and $\sigma^{\prime}=\sqrt{\vaa_{\ddist}X^{\prime}}$, whenever $k = \lceil \log(\delta^{-1})\rceil$ and $n \geq 2(1+\log(\delta^{-1}))$.

\item $\textup{\texttt{Cat}:}$ with $c \leq 2$ and $\sigma^{\prime}=\sqrt{\vaa_{\ddist}X^{\prime}}$, whenever $n \geq 4\log(\delta^{-1})$.

\item $\textup{\texttt{LM}:}$ with $c \leq 9\sqrt{2}$ and $\sigma^{\prime}=\sqrt{\vaa_{\ddist}X^{\prime}}$, whenever $n \geq (16/3)\log(8\delta^{-1})$.

\item $\textup{\texttt{Hol}:}$ with $c \leq \sqrt{2}$ and $\sigma^{\prime}=\sqrt{\exx_{\ddist}(X^{\prime})^{2}}$.
\end{itemize}
\end{lem}
\begin{proof}[Proof of Lemma \ref{lem:robust_mean_subroutine}]
All of these estimators require finite second moments, which trivially holds as $\exx_{\ddist}(X^{\prime})^{2} \leq \exx_{\ddist}X^{2} < \infty$ by our assumptions on $\ddist$. For the median-of-means estimator $\texttt{MoM}$, see \citet[Sec.~4.1]{devroye2016a} or \citet{hsu2016a} for a proof. For the Catoni-type estimator $\texttt{Cat}$, see \citet[Prop.~2.4]{catoni2012a} for a proof and characteristics of $s$ and $\rho^{\prime}$. For the truncated mean estimator $\texttt{LM}$, see the discussion and proofs from \citet[Thm.~1]{lugosi2019a} and \citet[Thm.~6]{lugosi2019b} for settings of $a$ and $b$. For the soft truncation estimator $\texttt{Hol}$, see \citet[Prop.~4]{holland2020a} for a proof and required properties of $\psi$ and $s$.
\end{proof}
\noindent The preceding lemma settles any issues regarding the availability of a sufficiently accurate sub-routine $\rmean$ under potentially heavy-tailed data. The key problem that remains is the fact that $\sigma^{\prime}$ depends on $\vvhat_{\alpha}$, and thus the second sample $\Y_{n}$. To remove this dependence, the following lemma will be useful (proof given in Appendix).
\begin{lem}[Good $\Y_{n}$ event]\label{lem:vhat_nice_order}
Let the observations $\Y_{n}$ sorted in increasing order be denoted by $\Y_{n}^{\ast} \defeq \{Y^{\ast}_{i}\}_{i \in [n]}$, such that $Y^{\ast}_{1} \leq Y^{\ast}_{2} \leq \ldots \leq Y^{\ast}_{n}$. It follows that with probability no less than $1-2\exp(-3n\alpha/14)$ over the draw of $\Y_{n}$, we have that
\begin{align*}
\vv_{2\alpha} \leq Y^{\ast}_{(1-\alpha)n} \leq \vv_{\alpha/2}.
\end{align*}
\end{lem}
\noindent Using the preceding lemma and setting $\vvhat_{\alpha} = Y^{\ast}_{(1-\alpha)n}$, we have
\begin{align}
\nonumber
\vaa_{\ddist}X^{\prime} = \vaa_{\ddist} X \, I_{\{X \geq \vvhat_{\alpha}\}} & = \exx_{\ddist}X^{2}I_{\{X \geq \vvhat_{\alpha}\}} - \left(\exx_{\ddist}X \, I_{\{X \geq \vvhat_{\alpha}\}}\right)^{2}\\
\nonumber
& \leq \sigma_{\alpha}^{2}\\
\label{eqn:variance_bound}
& \defeq \exx_{\ddist} X^{2}I_{\{X \geq V_{2\alpha}\}} - \left(\exx_{\ddist}X \, I_{\{X \geq \vv_{\alpha/2}\}}\right)^{2}.
\end{align}
As such, conditioning on $\Y_{n}$ and assuming that the good event of Lemma \ref{lem:vhat_nice_order} holds, then using variance bound (\ref{eqn:variance_bound}) and Lemma \ref{lem:robust_mean_subroutine} for $\cchat_{\alpha}^{\prime}$ given by (\ref{eqn:cchat_defn}), writing $\varepsilon(n,\delta) \defeq \sqrt{(1+\log(\delta^{-1}))/n}$ for readability, it follows that
\begin{align*}
\prr\left\{ | \cchat_{\alpha}^{\prime} - \exx_{\ddist}X^{\prime} | > c \sigma_{\alpha} \, \varepsilon(n,\delta) \right\} & \leq \prr\left\{ | \cchat_{\alpha}^{\prime} - \exx_{\ddist}X^{\prime} | > c \sigma^{\prime} \, \varepsilon(n,\delta) \right\} \leq \delta,
\end{align*}
assuming that we use any of the first three methods listed in Lemma \ref{lem:robust_mean_subroutine}, since $\sigma^{\prime}=\sqrt{\vaa_{\ddist}X^{\prime}}$. Otherwise, setting $\sigma_{\alpha}^{2} = \exx_{\ddist}X^{2}$ will suffice. The bound (\ref{eqn:variance_bound}) is useful since this gives us an upper bound which does not depend on the sample $\Y_{n}$. Stated more precisely, over the random draw of $\X_{n}$, we have
\begin{align}
\label{eqn:deviation_bd_Yfree}
\left|\alpha\,\cchat_{\alpha} - \exx_{\ddist} X \, I_{\{X \geq \vvhat_{\alpha}\}}\right| = | \cchat_{\alpha}^{\prime} - \exx_{\ddist}X^{\prime} | \leq c \sigma_{\alpha} \sqrt{\frac{1+\log(\delta^{-1})}{n}}
\end{align}
with probability no less than $1-\delta$.

Next, we consider the right-most summand in (\ref{eqn:decomposition}). This amounts to the error that must be incurred for not knowing $\vv_{\alpha}$ exactly. To control this term, first observe that
\begin{align*}
\exx_{\ddist} X \left( I_{\{X \geq \vv_{\alpha}\}} - I_{\{X \geq \vvhat_{\alpha}\}} \right) & \leq \exx_{\ddist} \vvhat_{\alpha} \left( I_{\{X \geq \vv_{\alpha}\}} - I_{\{X \geq \vvhat_{\alpha}\}} \right)\\
& \leq \vv_{\alpha/2}\left( \ddist\left\{ X \geq \vv_{\alpha} \right\} - \ddist\left\{ X \geq \vvhat_{\alpha} \right\} \right)\\
& = \vv_{\alpha/2}\left(F_{\ddist}(\vvhat_{\alpha})-F_{\ddist}(\vv_{\alpha}) \right)\\
& \leq \vv_{\alpha/2} \parasm \left(\vvhat_{\alpha}-\vv_{\alpha}\right).
\end{align*}
The first inequality is immediate from the events attached to the two indicators being subtracted. The second inequality uses the good event of Lemma \ref{lem:vhat_nice_order}. The final inequality uses the local $\parasm$-Lipschitz property via $\text{\ref{asmp:cdf_growth}}(\parasuf,\parasm)$. The problem has thus been reduced to obtaining two-sided bounds on the deviations $\vvhat_{\alpha}-\vv_{\alpha}$, which can be done easily using standard concentration properties of the empirical distribution function, as follows. Based on sample $\Y_{n}$, denote the empirical distribution function by $\widehat{F}_{n}(u) \defeq n^{-1}\sum_{i=1}^{n} I_{\{ Y_{i} \leq u \}}$, for $u \in \RR$. Considering the running assumption that $\vvhat_{\alpha} = Y_{(1-\alpha)n}^{\ast}$, note that for any error level $0 < \varepsilon \leq 1$, if the deviations are $\vvhat_{\alpha}-\vv_{\alpha} > \varepsilon$, then we must have $\widehat{F}_{n}(\vv_{\alpha}+\varepsilon) \leq 1-\alpha = F_{\ddist}(\vv_{\alpha})$. It then follows that
\begin{align*}
\prr\left\{ \vvhat_{\alpha}-\vv_{\alpha} > \varepsilon \right\} & \leq \prr\left\{ \widehat{F}_{n}(\vv_{\alpha}+\varepsilon) \leq F_{\ddist}(\vv_{\alpha}) \right\}\\
& = \prr\left\{ F_{\ddist}(\vv_{\alpha}+\varepsilon)-F_{\ddist}(\vv_{\alpha}) \leq F_{\ddist}(\vv_{\alpha}+\varepsilon)-\widehat{F}_{n}(\vv_{\alpha}+\varepsilon) \right\}\\
& \leq \prr\left\{ F_{\ddist}(\vv_{\alpha}+\varepsilon)-F_{\ddist}(\vv_{\alpha}) \leq \sup_{u \in \RR}\left[F_{\ddist}(u)-\widehat{F}_{n}(u)\right] \right\}\\
& \leq \exp\left(-2n(F_{\ddist}(\vv_{\alpha}+\varepsilon)-F_{\ddist}(\vv_{\alpha}))^{2}\right)\\
& \leq \exp\left(-2n(\parasuf\varepsilon)^{2}\right).
\end{align*}
The first three lines are immediate from the facts just stated. The exponential tail bound is the refined version of Dvoretzky-Kiefer-Wolfowitz (DKW) inequality, which holds even if $F_{\ddist}$ has at most a countably infinite number of discontinuities \citep[Thm.~11.6]{kosorok2008EPSP}. The final inequality is due to the $\parasuf$-growth assumption. For lower bounds, note that if $\vv_{\alpha}-\vvhat_{\alpha} > \varepsilon$, we must have $\widehat{F}_{n}(\vv_{\alpha}-\varepsilon) \geq 1-\alpha = F_{\ddist}(\vv_{\alpha})$, and a perfectly symmetric argument yields identical bounds on the probability of $\{\vv_{\alpha}-\vvhat_{\alpha} > \varepsilon\}$. Taking a union bound over these two events, it follows that with probability no less than $1-2\exp(-2n(\gamma\varepsilon)^{2})$, we have
\begin{align*}
\left| \exx_{\ddist} X \left( I_{\{X \geq \vv_{\alpha}\}} - I_{\{X \geq \vvhat_{\alpha}\}} \right) \right| \leq \vv_{\alpha/2} \parasm |\vvhat_{\alpha}-\vv_{\alpha}| \leq \vv_{\alpha/2} \parasm \varepsilon,
\end{align*}
for any $0 < \varepsilon \leq 1$. Converting this into a high-probability confidence interval, we have
\begin{align}
\label{eqn:exx_bd_Yfree}
\left| \exx_{\ddist} X \left( I_{\{X \geq \vv_{\alpha}\}} - I_{\{X \geq \vvhat_{\alpha}\}} \right) \right| \leq \frac{\vv_{\alpha/2}\parasm}{\sqrt{2}\parasuf} \sqrt{\frac{\log(\delta^{-1})}{n}}
\end{align}
with probability no less than $1-2\delta$, assuming that $n \geq \log(\delta^{-1})/(2\parasuf^{2})$. Taking (\ref{eqn:deviation_bd_Yfree}) and (\ref{eqn:exx_bd_Yfree}) together, applied to (\ref{eqn:decomposition}), we have essentially proved the following result.
\begin{thm}\label{thm:error_bd_static}
For any confidence level $\delta \in (0,1)$ and risk level $0 < \alpha < 1/2$, assume that $\text{\ref{asmp:cdf_growth}}(\parasuf,\parasm)$ holds and $n \geq \log(\delta^{-1}) \max\{1/(2\parasuf)^{2}, 14/(3\alpha)\}$. Letting $\cchat_{\alpha}^{\prime}$ be the output of Algorithm \ref{algo:static}, and $\cchat_{\alpha} = \cchat_{\alpha}^{\prime} / \alpha$, with probability no less than $1-5\delta$, we have
\begin{align*}
\left|\cchat_{\alpha} - \cc_{\alpha}\right| \leq \frac{1}{\alpha}\left(c\sigma_{\alpha} + \frac{\vv_{\alpha/2}\parasm}{\sqrt{2}\parasuf} \right)\sqrt{\frac{1+\log(\delta^{-1})}{n}},
\end{align*}
where $c$ depends only on the choice of $\rmean$ (specified in Lemma \ref{lem:robust_mean_subroutine}).
\end{thm}
\begin{proof}[Proof of Theorem \ref{thm:error_bd_static}]
To prove this result simply involves sorting out the key facts presented above. The ``good'' event in the theorem statement is that in which both (\ref{eqn:deviation_bd_Yfree}) and (\ref{eqn:exx_bd_Yfree}) hold together. This condition can fail if even one of the following bad events takes place:
\begin{align*}
\EE_{1} & \defeq \left\{ \text{inequality (\ref{eqn:subg_estimator}) fails} \right\}\\
\EE_{2} & \defeq \left\{ \text{event of Lemma \ref{lem:vhat_nice_order} fails} \right\}\\
\EE_{3} & \defeq \left\{ |\vvhat_{\alpha}-\vv_{\alpha}| > \sqrt{\frac{\log(\delta^{-1})}{2\parasuf^{2}n}} \right\}.
\end{align*}
First of all, using Lemma \ref{lem:robust_mean_subroutine} and the deviation bounds given by (\ref{eqn:subg_estimator}), we have
\begin{align*}
\prr(\EE_{1}) = \exx_{\Y_{n}}\prr\left[ \EE_{1} \,|\, \Y_{n} \right] \leq \delta.
\end{align*}
Next, by Lemma \ref{lem:vhat_nice_order}, if $n \geq 14\log(\delta^{-1})/(3\alpha)$, then we have $\prr(\EE_{2}) \leq 2\delta$. Finally, by the two-sided DKW inequality, whenever $n \geq \log(\delta^{-1})/(2\parasuf^{2})$, we have $\prr(\EE_{3}) \leq 2\delta$. If none of these three bad events take place, the good event holds, i.e., $(\EE_{1} \cap \EE_{2} \cap \EE_{3})^{c} \subseteq \{ \text{(\ref{eqn:deviation_bd_Yfree}) and (\ref{eqn:exx_bd_Yfree})} \}$. A union bound implies that this holds with probability no less than $1-4\delta$, and via the original decomposition (\ref{eqn:decomposition}), we have
\begin{align*}
\left|\cchat_{\alpha} - \cc_{\alpha}\right| \leq \frac{1}{\alpha}\left( c \sigma_{\alpha} \sqrt{\frac{1+\log(\delta^{-1})}{n}} + \frac{\vv_{\alpha/2}\parasm}{\sqrt{2}\parasuf} \sqrt{\frac{\log(\delta^{-1})}{n}} \right),
\end{align*}
which implies the desired result.
\end{proof}

\subsubsection{Comparison of estimation error bounds}\label{sec:bound_compare}

From the technical literature on CVaR estimation under potentially heavy-tailed data, the work of \citet{kolla2019a}, \citet{prashanth2019a}, and \citet{kagrecha2020a} are most closely related to our work, and in this remark we compare our results with theirs. To align our setup with theirs, we assume access to only $n$ data points in total, meaning the two data sets used in Theorem \ref{thm:error_bd_static} will now be $\X_{n/2}$ and $\Y_{n/2}$, for simplicity assuming that $n$ is even. Furthermore, we convert our high-confidence interval into an exponential tail bound, which is the form taken by the main results in the cited works. First, given just $n$ observations, our Theorem \ref{thm:error_bd_static} implies that
\begin{align*}
\prr\left\{ \left| \cchat_{\alpha}-\cc_{\alpha} \right| > \varepsilon \right\} & \leq 5\exp\left(-n\left(\alpha\varepsilon/B_{\textsc{ours}}\right)^{2}\right),\\
B_{\textsc{ours}} & \defeq c \sigma_{\alpha} + \frac{\sqrt{2}\vv_{\alpha/2}\parasm}{\parasuf}.
\end{align*}
The estimator $\cchat_{\alpha}$ considered by \citet[Thm.~4.1]{prashanth2019a}, on the other hand, yields bounds of the form
\begin{align*}
\prr\left\{ \left| \cchat_{\alpha}-\cc_{\alpha} \right| > \varepsilon \right\} & \leq 8\exp\left(-n\left(\alpha\varepsilon/B^{\prime}\right)^{2}\right),
\end{align*}
where the factor $B^{\prime}$ is simply left as a ``distribution-dependent factor.'' Looking at their proof, in order to obtain concentration of the VaR estimator, they also effectively require a $\parasuf$-growth property and have moment dependence. Furthermore, their proof is rather specialized to an estimator borrowed from \citet{bubeck2013a}, which does random truncation that is rather unintuitive when taken outside the context of online learning problems. Another closely related result published very recently is due to \citet{kagrecha2020a}. They consider a more natural estimator, which simply truncates the data to $|X_{i}| \leq b$ before passing it to the classical empirical CVaR estimator routine. While $b$ is a user-specified parameter, it must be taken larger than a value which depends on the desired deviation level $\varepsilon$. In particular, since it must satisfy $b = \Omega(\exx_{\ddist}X^{2}/(\alpha\varepsilon))$, when $\varepsilon$ is sufficiently small, one ends up with bounds of the form
\begin{align*}
\prr\left\{ \left| \cchat_{\alpha}-\cc_{\alpha} \right| > \varepsilon \right\} & \leq 6\exp\left(-n\alpha^{3}\varepsilon^{4}/B^{\prime\prime}\right),\\
B^{\prime\prime} & \defeq 616 \left(\exx_{\ddist}X^{2}\right)^{2}.
\end{align*}
Their results are obtained using very weak assumptions, the finiteness of $\exx_{\ddist}X^{2}$ is all that is required. The price paid for this generality is clearly the poor dependence on $\alpha$, $\varepsilon$, and the moments. In contrast, under mild additional assumptions on the behaviour of the distribution function around $\vv_{\alpha}$ (namely $\text{\ref{asmp:cdf_growth}}(\parasuf,\parasm)$), we obtain much stronger results, using a very simple proof strategy, which can be readily applied to a wide collection of estimation routines.

\subsection{CVaR-driven learning algorithms}\label{sec:theory_dynamic}

We now proceed to our main point of interest, namely learning algorithms which seek to minimize the CVaR of the loss distribution, defined in (\ref{eqn:cc_defn}), given only a sample $\Z_{n} \defeq \{Z_{1},\ldots,Z_{n}\}$, independent copies of $Z \sim \ddist$. Computationally, it is convenient to introduce
\begin{align}
f_{\alpha}(w,v;Z) \defeq v + \frac{1}{\alpha}\left[\loss(w;Z)-v\right]_{+}, \qquad w \in \WW, v \in \RR
\end{align}
with expected value denoted by $F_{\alpha}(w,v) \defeq \exx_{\ddist}f_{\alpha}(w,v;Z)$, not to be confused with $F_{\ddist}$ from the previous section. This expectation has the useful property of being convex and continuously differentiable in $v$, and being related to the quantities $\cc_{\alpha}(w)$ and $\vv_{\alpha}(w)$ through
\begin{align*}
\min \{F_{\alpha}(w,v): v \in \RR\} = F_{\alpha}(w,\vv_{\alpha}(w)) = \cc_{\alpha}(w),
\end{align*}
which holds for any choice of $w \in \WW$ \citep[Thm.~1]{rockafellar2000a}. This implies that if we have some candidates $(\what,\vhat)$ such that $F_{\alpha}(\what,\vhat) \leq \varepsilon$, then $\cc_{\alpha}(\what) \leq F_{\alpha}(\what,\vhat) \leq \varepsilon$. Furthermore, solving the joint problem is equivalent to solving the two problems separately \citep[Thm.~2]{rockafellar2000a}, meaning that $F_{\alpha}^{\ast} = \cc_{\alpha}^{\ast}$, where we denote $F_{\alpha}^{\ast} \defeq \inf\{F_{\alpha}(w,v): (w,v) \in \WW \times \RR\}$, $\cc_{\alpha}^{\ast} \defeq \inf\{\cc_{\alpha}(w): w \in \WW\}$. When $\loss(w;Z)$ is convex in $w$, the function $F_{\alpha}$ is jointly convex in its arguments, and thus when $\WW \subseteq \RR^{d}$ is a convex set, convex optimization techniques can in principle be brought to bear on the problem. Of course in practice, this is a learning problem and the underlying distribution $\ddist$ is never known.\footnote{This is also known as a stochastic convex optimization problem, and there is a rich literature on the subject. See the references given by \citet[Sec.~2]{rockafellar2000a}.} The traditional machine learning approach to this is empirical risk minimization, namely returning any
\begin{align}
(\what_{\erm},\vhat_{\erm}) \in \argmin_{(w,v) \in \WW \times \RR} \frac{1}{n} \sum_{i=1}^{n} f_{\alpha}(w,v;Z_{i}).
\end{align}
While the objective function is not differentiable everywhere, sub-gradients can be readily computed, and descent methods using sub-gradients can be applied to implement this optimization \citep[Sec.~4]{rockafellar2000a}. On the statistical side, however, under potentially heavy-tailed losses, only highly sub-optimal performance guarantees can be given in general for $\what_{\erm}$ \citep{brownlees2015a}, which motivates the need for providing the learner with ``better feedback.''

\paragraph{Problems with robust objectives}
Recalling the analysis of the previous section \ref{sec:theory_static}, we constructed a procedure for obtaining sharp estimates of $\cc_{\alpha}(w)$, pointwise in $w$, under potentially heavy-tailed data. To extend the procedure given by Algorithm \ref{algo:static} and defined in (\ref{eqn:cchat_defn}) to this setting, one could naturally split the sample $\Z_{n}$, compute
\begin{align}\label{eqn:cchat_w_defn}
\cchat_{\alpha}^{\prime}(w;\Z_{n}) \defeq \cchat_{\alpha}^{\prime}\left[\X=\left\{\loss(w;Z_{i}): i \in [\lfloor n/2 \rfloor]\right\}, \Y=\left\{\loss(w;Z_{i}): n/2 < i \leq n \right\}\right],
\end{align}
and set $\cchat_{\alpha}(w) = \cchat_{\alpha}^{\prime}(w;\Z_{n}) / \alpha$. For any candidate $w \in \WW$, the approximation $\cchat_{\alpha}(w) \approx \cc_{\alpha}(w)$ is accurate with high confidence, as formalized in Theorem \ref{thm:error_bd_static}. This can naturally be interpreted as feedback to the learner which is ``robust'' to potentially heavy-tailed data. The most naive approach to this problem would be to replace the empirical mean with this robust estimator (\ref{eqn:cchat_w_defn}), namely any algorithm implementing
\begin{align*}
\what \in \argmin_{w \in \WW} \cchat_{\alpha}^{\prime}(w;\Z_{n}) / \alpha.
\end{align*}
The statistical properties of such an $\what$ are naturally of interest, but the computational task of actually obtaining such a $\what$ is highly non-trivial; for example the work of \citet{brownlees2015a} consider a similar quantity in the case of traditional risk minimization, but algorithmic considerations are left completely abstract. Indeed, even if $\loss(\cdot,z)$ is convex for all $z \in \ZZ$, we have no guarantee that $\cchat_{\alpha}^{\prime}(\cdot;\Z_{n})$ will be. The exact same issues hold if we tackle a robustified version of the joint optimization task, namely
\begin{align*}
(\what,\vhat) \in \argmin_{(w,v) \in \WW \times \RR} \rmean\left[ \left\{ f_{\alpha}(w,v;Z_{i}): i \in [n] \right\} \right],
\end{align*}
where $\rmean$ is based on any procedure given in Lemma \ref{lem:robust_mean_subroutine}. All the robust estimates given by $\rmean$ (or Algorithm \ref{algo:static}) are easy to \emph{compute} for any $(w,v)$ or $w$, but are hard to \emph{minimize}. It thus seems wiser to use such sub-routines for \emph{validation}, i.e., to check that a particular candidate $\what$ actually gets close to minimizing $\cc_{\alpha}(\cdot)$ with sufficiently high confidence.

\paragraph{A more practical approach}
With this intuition in mind, we present a procedure which utilizes the insights of section \ref{sec:theory_static} to obtain strong statistical guarantees, without sacrificing computational efficiency. In words, we consider a simple divide-and-conquer procedure with independent sub-processes running stochastic gradient descent for the joint optimization of $F_{\alpha}$, and a final robust validation step to determine a final candidate. This is summarized in Algorithm \ref{algo:dynamic}, and we unpack the notation below.

\begin{algorithm}[t!]
\caption{Fast gradient-based CVaR learning with robust verification.}
\label{algo:dynamic}
\begin{algorithmic}
\State \textbf{inputs:} samples $\Z_{n}$ and $\Z_{n}^{\prime}$, initial value $(\what_{0},\vhat_{0})$, parameters $\alpha \in (0,1)$, $0 < \vv < \infty$, $1 \leq k \leq n$.

\medskip

\State Split $\displaystyle \bigcup_{j = 1}^{k}\II_{j} = \displaystyle [n]$, with $|\II_{j}| \geq \lfloor n/k \rfloor$, and $\II_{j} \cap \II_{l} = \emptyset$ when $j \neq l$. \hfill\Comment{Disjoint partition.}

\medskip

\State For each $j \in [k]$, set $\displaystyle (\overbar{w}^{(j)},\overbar{v}^{(j)})$ to the mean of sequence $\displaystyle \SGD(\what_{0},\vhat_{0};\Z_{\II_{j}},\WW \times [0,\vv])$.

\medskip

\State Compute $\displaystyle \star = \argmin_{j \in [k]} \,  \cchat_{\alpha}^{\prime}\left(\overbar{w}^{(j)}; \Z_{n}^{\prime}\right)$. \hfill\Comment{Robust validation via Algorithm \ref{algo:static}.}

\medskip

\State \textbf{return} $\displaystyle \overbar{w}^{(\star)}$.
\end{algorithmic}
\end{algorithm}

Most of the steps in Algorithm \ref{algo:dynamic} are transparent; in the core validation step, we pass the sub-routine defined in (\ref{eqn:cchat_w_defn}) its own independent sample $\Z_{n}^{\prime}$. It just remains to provide a more precise definition of the $\SGD$ sequence referred to in the third line. Given a sequence of observations $(Z_{1},\ldots,Z_{t})$ of arbitrary length $t \geq 1$, the core update is traditional projected stochastic sub-gradient descent:
\begin{align}\label{eqn:sgd_defn}
(\what_{t},\vhat_{t}) = \proj_{\WW \times [0,\vv]}\left[ (\what_{t-1},\vhat_{t-1}) - \beta_{t} \, G_{\alpha}(\what_{t-1},\vhat_{t-1};Z_{t}) \right]
\end{align}
The update direction here is $G_{\alpha}(w,v;Z) \in \partial f_{\alpha}(w,v;Z)$, namely any vector from the sub-differential of the map $(w,v) \mapsto f_{\alpha}(w,v;Z)$. The operator $\proj$ denotes projection in the $\ell_{2}$ norm, and $\beta_{t} \geq 0$ is a step-size parameter. The recursive definition in (\ref{eqn:sgd_defn}) bottoms out at $t=1$, and is initialized by some pre-defined $(\what_{0},\vhat_{0})$, passed to the algorithm as an input. The sequence $\SGD(\what_{0},\vhat_{0};\Z_{\II_{j}},\WW \times [0,\vv])$ referred to in Algorithm \ref{algo:dynamic} is simply the sequence of iterates generated by (\ref{eqn:sgd_defn}) using data $\{Z_{t}: t \in \II_{j} \}$; since all $Z_{t}$ are independent copies of $Z \sim \ddist$, the order does not matter. The key technical assumptions on the data are summarized below:
\begin{itemize}
\item[\namedlabel{asmp:data}{A2}.] Let $\text{\ref{asmp:cdf_growth}}(\parasuf,\parasm)$ hold for $X = \loss(w;Z) \geq 0$, for any choice of $w \in \WW$.  Let $\WW$ be convex, have a diameter in $\ell_{2}$ norm of $0 < \diameter < \infty$. Let $\overbar{\sigma}_{\alpha} \defeq \max\{\sigma_{\alpha}(w): w \in \WW\} < \infty$ and $\overbar{\vv}_{\alpha} \defeq \max\{\vv_{\alpha}(w): w \in \WW\} < \infty$. Let $\loss(w;z)$ be a convex, $\parasm_{\loss}$-Lipschitz continuous function of $w$, for all $z \in \ZZ$.
\end{itemize}
\noindent The preceding assumptions clearly allow for potentially heavy-tailed losses. Note $\sigma_{\alpha}(w)$ extends $\sigma_{\alpha}$ from section \ref{sec:theory_static} to the case of $X = \loss(w;Z)$. Under this setting, the following performance guarantee holds.
\begin{thm}\label{thm:error_bd_dynamic}
Under assumption $\text{\ref{asmp:data}}$, run Algorithm \ref{algo:dynamic} with parameters $0 < \alpha < 1/2$, $V = \overbar{\vv}_{\alpha}$, $k = \lceil \log(2\lceil\log(\delta^{-1})\rceil \delta^{-1}) \rceil$ for arbitrary choice of $\delta \in (0,1)$, and fix the step sizes in (\ref{eqn:sgd_defn}) to 
\begin{align*}
\beta_{t} = \alpha\sqrt{\frac{\diameter^{2}+\overbar{\vv}_{\alpha}}{(\parasm_{\loss}^{2}+(1-\alpha)^{2})|\II_{j}|}}
\end{align*}
for each sub-process, indexed by $j \in [k]$. We have
\begin{align}
\cc_{\alpha}(\overbar{w}^{(\star)})-\cc_{\alpha}^{\ast} \leq \frac{2\sqrt{2}}{\alpha}\left(c\overbar{\sigma}_{\alpha} + \frac{\overbar{\vv}_{\alpha/2}\parasm}{\sqrt{2}\parasuf} \right)\sqrt{\frac{1+\log(5\delta^{-1})}{n}} + \frac{\ct{e}}{\alpha}\sqrt{\frac{k(\parasm_{\loss}^{2}+(1-\alpha)^{2})(\diameter^{2}+\overbar{\vv}_{\alpha}^{2})}{n}}
\end{align}
with probability no less than $1-3\delta$, where constant $c$ corresponds to those in Lemma \ref{lem:robust_mean_subroutine}.
\end{thm}
\begin{rmk}[Discussion of related technical work]\label{rmk:compare_algos}
As far as technical conditions go, the convexity, bounded diameter, and Lipschitz assumptions align with \citet[Thm.~3.6]{soma2020a}. They run a single averaged SGD process using a surrogate objective, for multiple passes over the data; they assume bounded losses and Lipschitz-continuous gradients, yielding error bounds in expectation. In contrast, we do not require Lipschitz gradients, the losses can be unbounded (and potentially heavy-tailed of course), and we run multiple SGD processes in parallel, each of which takes only a single pass over the subset of data allocated to it. Finally, we remark that since their procedure does not actually make any direct estimates of $\vv_{\alpha}$, they do not use an assumption like $\text{\ref{asmp:cdf_growth}}$. Note that it is certainly possible to modify our Algorithm \ref{algo:dynamic} such that this assumption is not needed, by doing the final validation step based on an estimate of $F_{\alpha}$ instead of $\cc_{\alpha}$. This would remove the need for $\text{\ref{asmp:cdf_growth}}$, and instead result in bounds depending on the second moment of $f_{\alpha}(w,v;Z)$. The formal analysis goes through in a perfectly analogous fashion to our proof of Theorem \ref{thm:error_bd_dynamic} here. We leave empirical analysis of such an alternative procedure to future work.\hfill$\blacksquare$
\end{rmk}

Proving the preceding theorem just requires combining a few basic techniques and structural results. To open up the argument, note that for any choice of $w \in \WW$ and $v \in \RR$, we can control the excess CVaR as
\begin{align}\label{eqn:error_bd_dynamic_1}
\cc_{\alpha}(w) - \cc_{\alpha}^{\ast} = \cc_{\alpha}(w) - F_{\alpha}^{\ast} \leq F_{\alpha}(w,v) - F_{\alpha}^{\ast}.
\end{align}
The equality and inequality follow respectively from Theorems 2 and 1 of \citet{rockafellar2000a}. Working on the right-hand side of this inequality, we can focus on (approximate) minimization of the function $F_{\alpha}$. While in principle this can be done in very sophisticated ways, for clarity of exposition, we adapt a well-known result for averaged stochastic gradient descent to the objective of interest here.
\begin{lem}[Convex, Lipschitz case; averaged SGD]\label{lem:learn_conv_lip_SGDave}
If the function $(w,v) \mapsto f_{\alpha}(w,v;z)$ is convex and $\parasm$-Lipschitz, consider running (\ref{eqn:sgd_defn}) for $m$ iterations, with fixed step size $\beta_{t} = \sqrt{(\diameter^{2}+V^{2})/m}/\parasm$. Then averaging the iterates as
\begin{align*}
(\what_{[m]},\vhat_{[m]}) \defeq \frac{1}{m} \sum_{t=1}^{m} (\what_{t-1},\vhat_{t-1}),
\end{align*}
it follows that in expectation over data $Z_{1},\ldots,Z_{m}$ that
\begin{align*}
\exx\left[ F_{\alpha}(\what_{[m]},\vhat_{[m]}) - F_{\alpha}^{\ast} \right] \leq \parasm\sqrt{\frac{\diameter^{2}+V^{2}}{m}}.
\end{align*}
\end{lem}
\noindent In order to utilize the preceding lemma, we simply need to confirm the required properties of $f_{\alpha}$, which we summarize in the following lemma.
\begin{lem}\label{lem:dynamic_lip}
Let $\WW \subseteq \RR^{d}$ be a convex set, and let the map $w \mapsto \loss(w;z)$ defined on $\WW$ be convex and $\parasm$-Lipschitz, for all values of $z \in \ZZ$. Then for any $0 < \alpha < 1$, writing
\begin{align*}
\parasm_{\alpha} \defeq \max\left\{ 1, \frac{\sqrt{\parasm^{2} + (1-\alpha)^{2}}}{\alpha} \right\},
\end{align*}
we have that for all $z \in \ZZ$, the map $(w,v) \mapsto f_{\alpha}(w,v;z)$ defined on $\WW \times \RR$ is convex and $\parasm_{\alpha}$-Lipschitz.
\end{lem}
\noindent Plugging in the content of Lemma \ref{lem:dynamic_lip} into Lemma \ref{lem:learn_conv_lip_SGDave}, we have that the sub-processes in Algorithm \ref{algo:dynamic} satisfy
\begin{align}\label{eqn:error_bd_dynamic_2}
\exx\left[ F_{\alpha}(\overbar{w}^{(j)},\overbar{v}^{(j)}) - F_{\alpha}^{\ast} \right] \leq \parasm_{\alpha}\sqrt{\frac{\diameter^{2}+\vv^{2}}{\lfloor n/k\rfloor}}, \qquad j \in [k].
\end{align}
Finally, we use the fact that robust validations of the form studied in section \ref{sec:theory_static} let us boost the confidence of the underlying SGD sub-processes \citep[Lemma 2]{holland2020nonsc}.
\begin{lem}[Boosting the confidence under potentially heavy tails]\label{lem:boost_conf_unbounded}
Assume that we have an arbitrary learning algorithm $\learn$, and a validation procedure $\valid$ such that for sample size $n \geq 1$, confidence level $\delta \in (0,1)$, and arbitrary $w \in \WW$, given samples $\Z_{n}$ and $\Z_{n}^{\prime}$, we have
\begin{align*}
\prr\left\{ \cc_{\alpha}(\learn\left[\Z_{n}\right])-\cc_{\alpha}^{\ast} > \frac{\varepsilon(n)}{\delta} \right\} & \leq \delta\\
\prr\left\{ |\valid\left[w;\Z_{n}^{\prime}\right]-\cc_{\alpha}(w)| > \varepsilon^{\prime}(n,\delta)\right\} & \leq \delta.
\end{align*}
Then, if we split the sample $\Z_{n}$ into $k$ disjoint subsets indexed by $\II_{1},\ldots,\II_{k}$, set $\what^{(j)} = \learn[\Z_{\II_{j}}]$ for each $j \in [k]$, and $\star = \argmin_{j \in k} \valid[\what^{(j)};\Z_{n}^{\prime}]$, then for any choice of $\delta \in (0,1)$, it follows that
\begin{align*}
\cc_{\alpha}(\what^{(\star)})-\cc_{\alpha}^{\ast} \leq 2\varepsilon^{\prime}(n,\delta) + \ct{e} \, \varepsilon\left(\left\lfloor\frac{n}{k}\right\rfloor\right)
\end{align*}
with probability no less than $1-k\delta-\ct{e}^{-k}$.
\end{lem}
\noindent With these facts in hand, it is straightforward to prove the desired theorem.
\begin{proof}[Proof of Theorem \ref{thm:error_bd_dynamic}]
Using inequality (\ref{eqn:error_bd_dynamic_1}) to connect $\cc_{\alpha}$ and $F_{\alpha}$, and Markov's inequality to convert the bounds in expectation for the sub-processes given by (\ref{eqn:error_bd_dynamic_2}) to high-probability bounds, it immediately follows that the requirement on $\learn$ in Lemma \ref{lem:boost_conf_unbounded} is satisfied if we set $\learn[\cdot]=\texttt{Average}[\SGD(\what_{0},\vhat_{0};\cdot,\WW \times [0,\vv])]$, with $\varepsilon(\lfloor n/k \rfloor)$ corresponding to the right-hand side of the inequality (\ref{eqn:error_bd_dynamic_2}), and $\texttt{Average}$ simply denoting taking the arithmetic vector mean. As for the requirement on $\valid$ in Lemma \ref{lem:boost_conf_unbounded}, this is satisfied by setting $\valid[w;\Z_{n}^{\prime}] = \cchat_{\alpha}^{\prime}(w;\Z_{n}^{\prime})$, as defined in (\ref{eqn:cchat_w_defn}), and $\varepsilon^{\prime}$ being controlled using Theorem \ref{thm:error_bd_static} with $X = \loss(w;Z)$, to obtain
\begin{align*}
\varepsilon^{\prime}\left(n,\delta\right) \leq \frac{\sqrt{2}}{\alpha}\left(c\sigma_{\alpha}(w) + \frac{\vv_{\alpha/2}(w)\parasm(w)}{\sqrt{2}\parasuf(w)} \right)\sqrt{\frac{1+\log(5\delta^{-1})}{n}}.
\end{align*}
Here $\sigma_{\alpha}(w)$ is given by (\ref{eqn:variance_bound}) with $X = \loss(w;Z)$, and $(\parasuf(w),\parasm(w))$ correspond to the parameters in $\text{\ref{asmp:cdf_growth}}$ applied to the distribution of $X=\loss(w;Z)$. Using $\text{\ref{asmp:data}}$, we bound all the $w$-dependent factors using $\parasm/\parasuf$, $\overbar{\sigma}_{\alpha}$, and $\overbar{\vv}_{\alpha/2}$. Also compared with the bound in Theorem \ref{thm:error_bd_static}, note the factor of $5$ in the logarithmic term used to get a $1-\delta$ confidence interval, and the $\sqrt{2}$ factor due to splitting the sample.

Placing things in the context of Algorithm \ref{algo:dynamic}, the concrete $\learn$ and $\valid$ procedures just described are precisely what Algorithm \ref{algo:dynamic} implements. As such, we can use Lemma \ref{lem:boost_conf_unbounded} and the bounds on $\varepsilon$ and $\varepsilon^{\prime}$ just discussed to get bounds on $\cc_{\alpha}(\overbar{w}^{(\star)})$ with probability no less than $1-k\delta-\ct{e}^{-k}$. To clean up this probability, let us specify the number of partitions carefully. Writing $k_{\delta} \defeq \lceil \log(\delta^{-1}) \rceil$ and $\delta^{\ast} \defeq \delta / 2k_{\delta}$, where $\delta \in (0,1)$ is the confidence parameter of Theorem \ref{thm:error_bd_dynamic}, set the number of partitions to be $k=k_{\delta^{\ast}} = \lceil \log(1/\delta^{\ast}) \rceil = \lceil \log(2\lceil\log(\delta^{-1})\rceil\delta^{-1}) \rceil$. It is straightforward to bound $k_{\delta^{\ast}} \delta^{\ast} \leq 2\delta$ and $\exp(-k_{\delta^{\ast}}) \leq \delta$ \citep{holland2020nonsc}, which gives probability of at least $1-3\delta$. Finally, the desired result follows from plugging $\parasm_{\loss}$ from $\text{\ref{asmp:data}}$ into the definition of $\parasm_{\alpha}$, and noting that $\parasm_{\alpha} \geq 1$ whenever $\alpha \leq 1/2$.
\end{proof}

\section{Empirical analysis}\label{sec:empirical}

In this section, we start with a numerical investigation of the efficiency of pointwise CVaR estimation enabled by the analysis of section \ref{sec:theory_static}, using concrete implementations of Algorithm \ref{algo:static}. This is followed by an empirical analysis of the performance of CVaR-driven learning algorithms, including Algorithm \ref{algo:dynamic} studied in section \ref{sec:theory_dynamic}.

\subsection{Accuracy of pointwise estimates}

\paragraph{Experimental setup}

Recalling the notation of section \ref{sec:theory_static}, given samples $\X_{n}$ and $\Y_{n}$, all sampled independently from $X \sim \ddist$, the objective here is to investigate the deviations $|\cchat_{\alpha} - \cc_{\alpha}|$, in particular how these deviations change for different estimators $\cchat_{\alpha}$, distributions $\ddist$, sample sizes $n$, and risk levels $\alpha$. For choice of $\ddist$, we test three distribution families: folded Normal, log-Normal, and Pareto. We have set these distributions such that the width of their inter-quartile range is approximately the same (fixed at $3.4$) for all choices of $\ddist$. We test a range of values for $n$ and $\alpha$. Each distinct experimental setting is characterized by the triplet $(\ddist,n,\alpha)$, and for each experimental setting, we run $10000$ independent trials, from which we obtain box-plots as well as the empirical average and standard deviation for $|\cchat_{\alpha} - \cc_{\alpha}|$. For $\cc_{\alpha}$, instead of using numerical integration, instead for each choice of $(\ddist,\alpha)$, we prepare two independent large samples from $\ddist$, each of size $n=10^{8}$, compute $\vv_{\alpha}$ as the empirical $(1-\alpha)$-level quantile on the first large sample, and $\cc_{\alpha}$ as $\sum_{i=1}^{n} X_{i}I_{\{X_{i} \geq \vv_{\alpha}\}} / (n\alpha)$ on the second large sample.

Regarding the methods being compared, all procedures estimate $\vv_{\alpha}$  in the same way, namely by sorting $\Y_{n}$ and using the $(1-\alpha)$-level quantile. The key differences are in how $\cchat_{\alpha}$ is computed. As baseline methods, we consider the classical empirical mean (denoted \texttt{Empirical}) and the random truncation method studied by \citet{prashanth2019a} (denoted \texttt{R-Trunc}). The latter depends on an upper bound ($u$ in their notation), which we set as the empirical mean of $\{X_{i}^{2}: i \in [n]\}$. To compare this with algorithms that newly fall under the scope of our analysis in section \ref{sec:theory_static}, we consider Algorithm \ref{algo:static} implemented using special cases \texttt{Cat} (denoted \texttt{Cat-12}) and \texttt{MoM} (denoted \texttt{MoM}) mentioned in Lemma \ref{lem:robust_mean_subroutine}. The former requires an empirical scale estimate, which we do using a standard M-estimate of dispersion, precisely following \citet{holland2019c} (and their online code). The latter requires the sample $\X_{n}$ to be split into $k$ independent subsets, and we set $k = 1+\lceil 3.5\log(\delta^{-1})\rceil$ following \citet[Algorithm 3]{prasad2018a}. All methods aside from \texttt{Empirical} depend on a confidence parameter $\delta$, which we set to $\delta = 0.02$.

\paragraph{Results and discussion}

Key results for the conditions described above are summarized in Figures \ref{fig:static_n} and \ref{fig:static_alpha}. Starting with Figure \ref{fig:static_n}, we see that ranging from small to large values of $n$, across all the distributions considered, the M-estimator approach (\texttt{Cat-12}) achieves a strong balance between robustness to outliers and bias, leading to superior performance on average with competitive variance. Moving to Figure \ref{fig:static_alpha}, we observe an analogous trend as we take $\alpha$ from large to small with a fixed sample size. In both settings, the bias of the other two robust methods leads to deviations that are worse on average than the naive empirical mean. As a general take-away, we see that using a slightly more sophisticated estimation procedure can lead to clear improvements in estimation in a potentially heavy-tailed setting. For our purposes, it is worth noting that the empirical procedure which performed best overall (\texttt{Cat-12}) is a procedure captured by the theory of section \ref{sec:theory_static}.

\begin{figure}[t]
\centering
\includegraphics[width=0.5\textwidth]{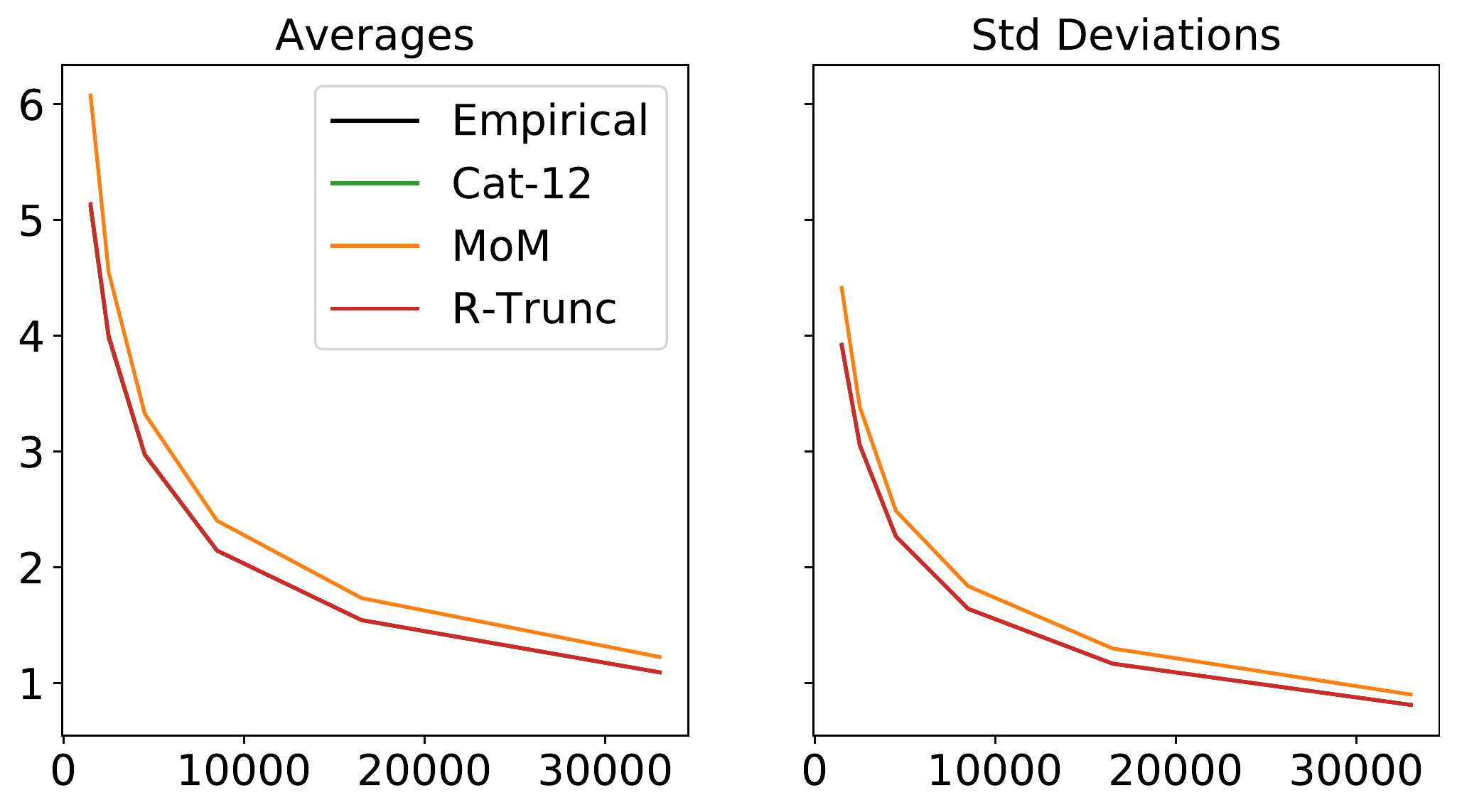}\,\includegraphics[width=0.5\textwidth]{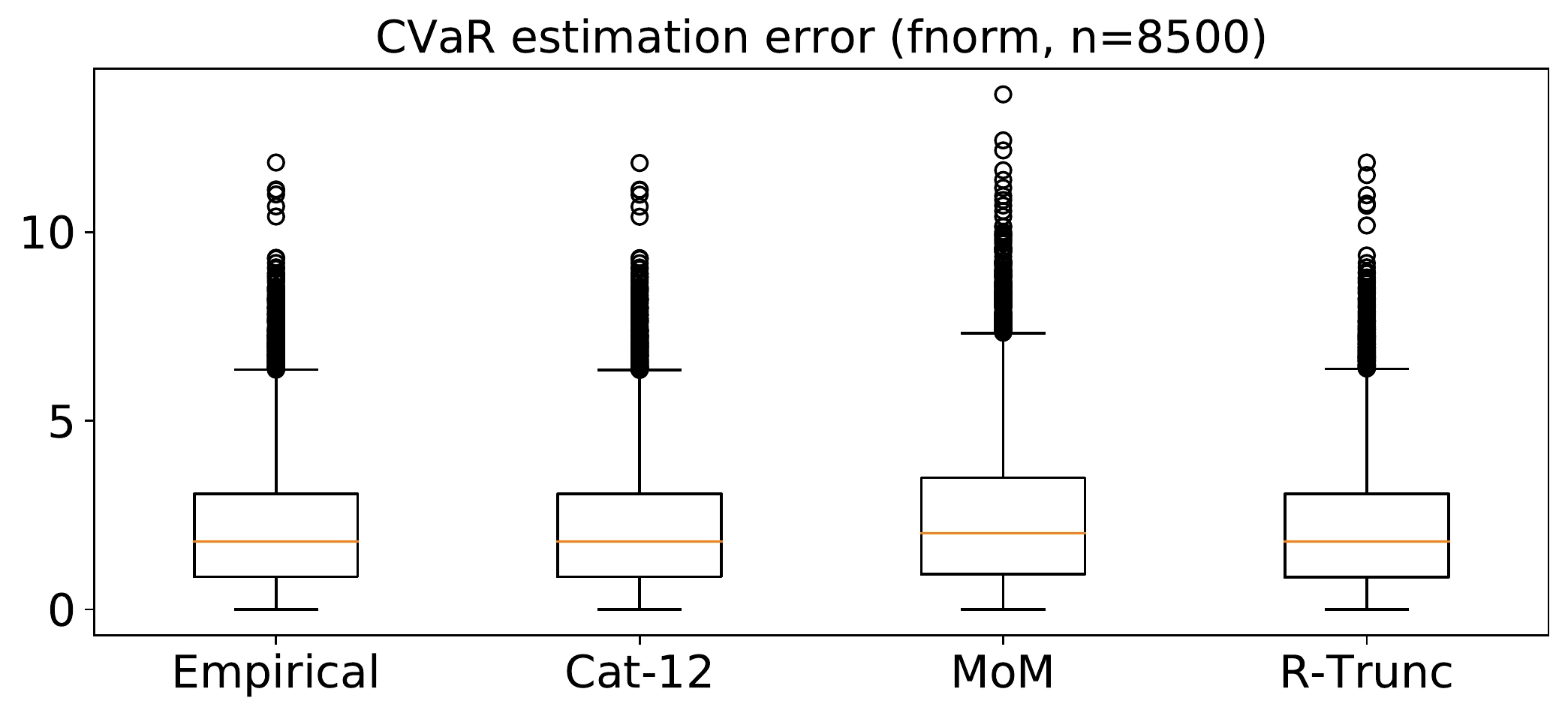}\\
\includegraphics[width=0.5\textwidth]{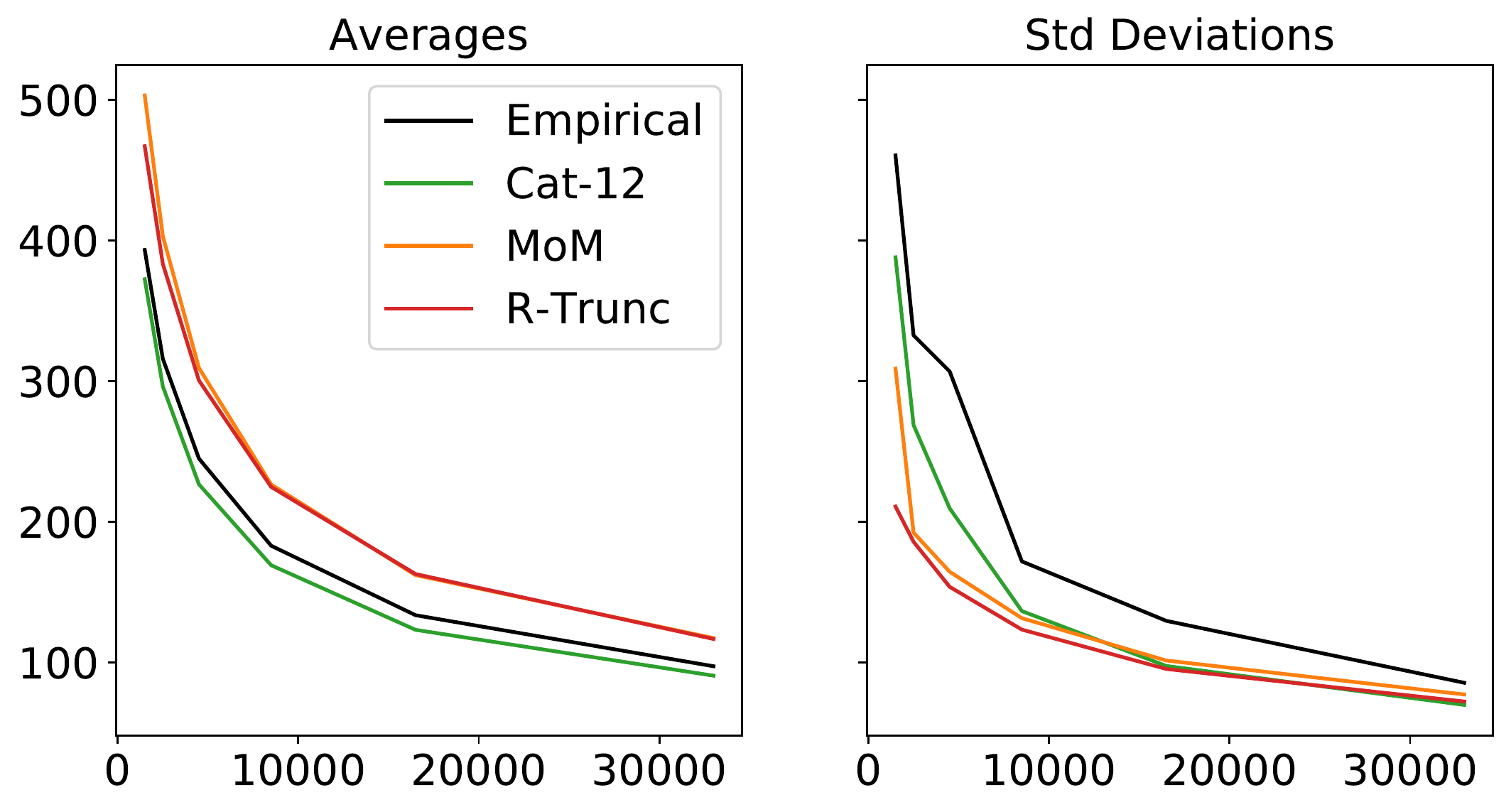}\,\includegraphics[width=0.5\textwidth]{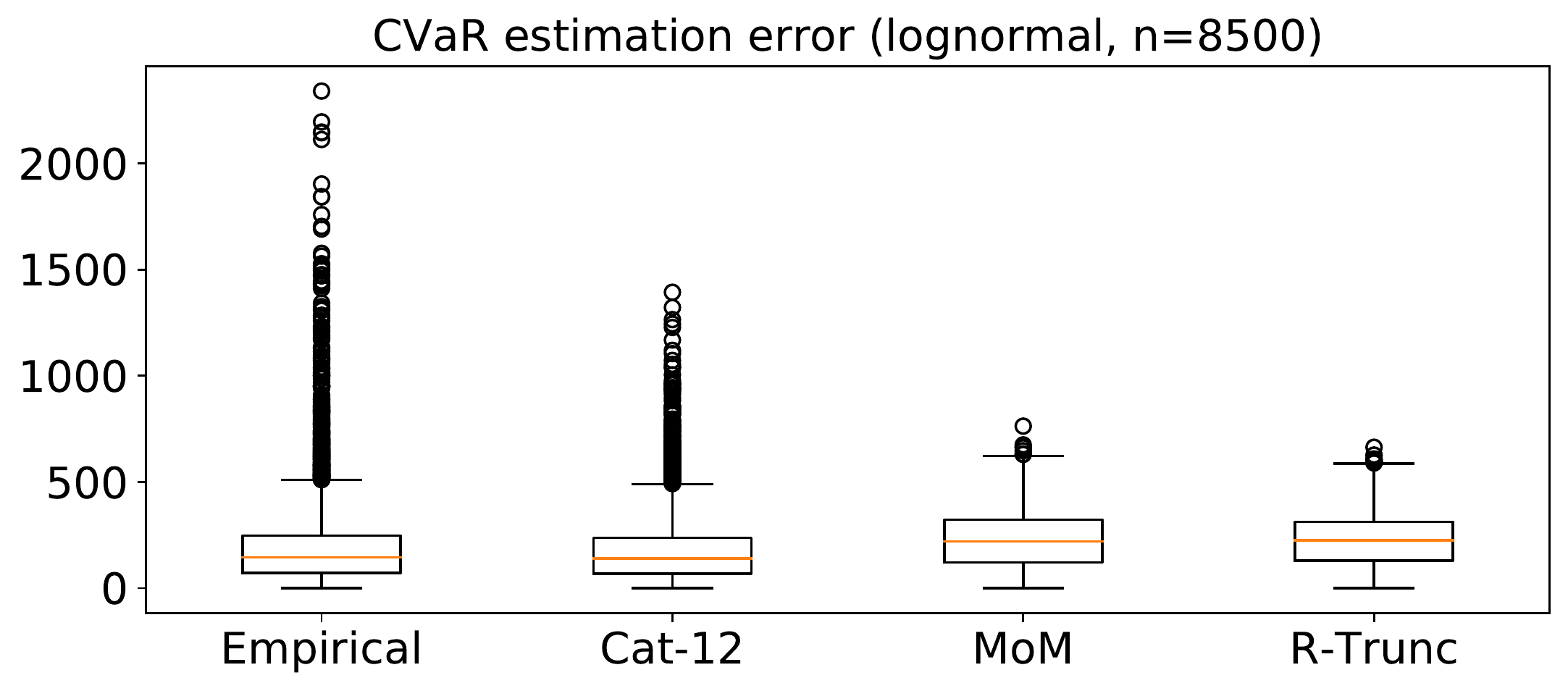}\\
\includegraphics[width=0.5\textwidth]{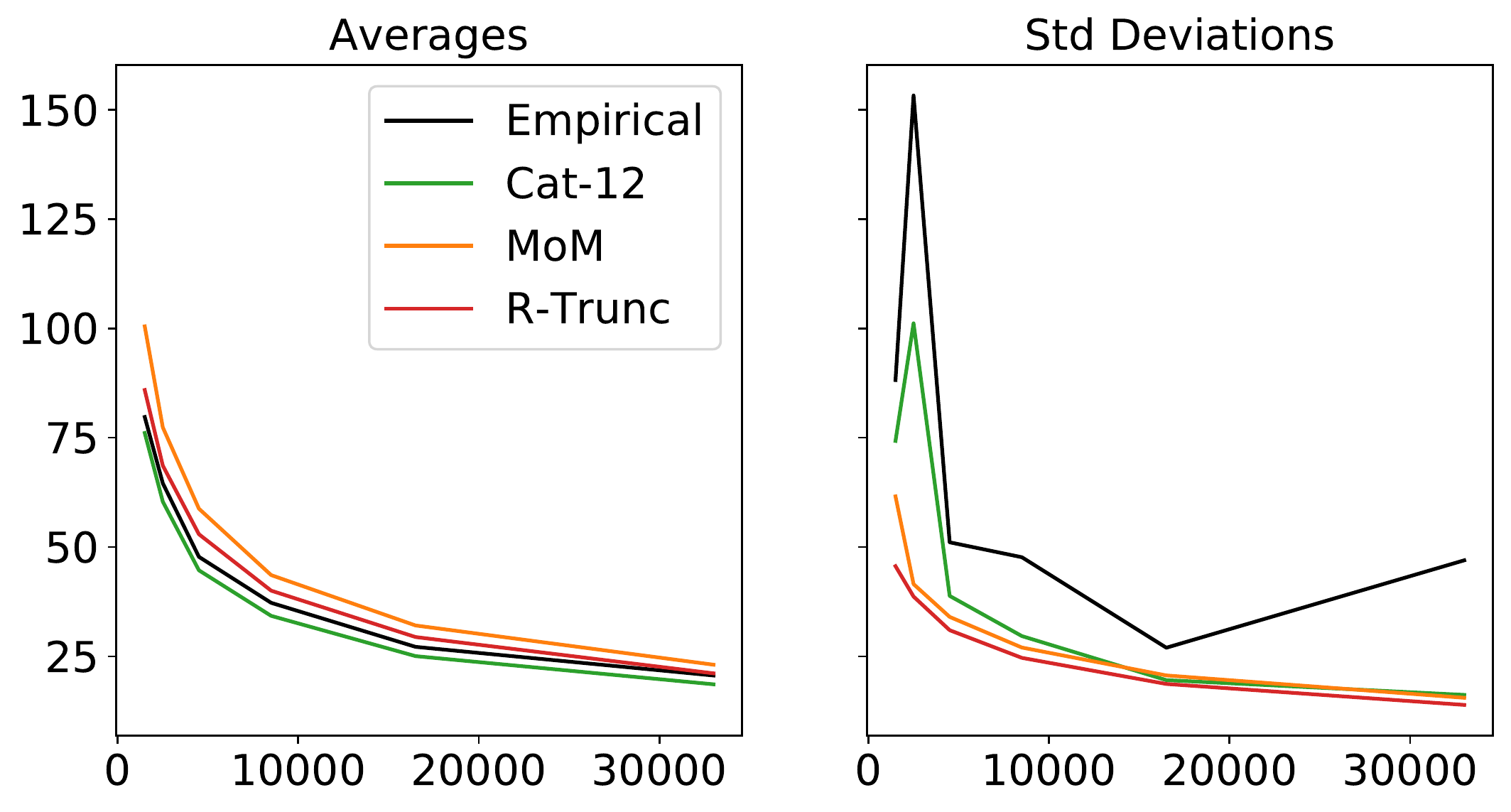}\,\includegraphics[width=0.5\textwidth]{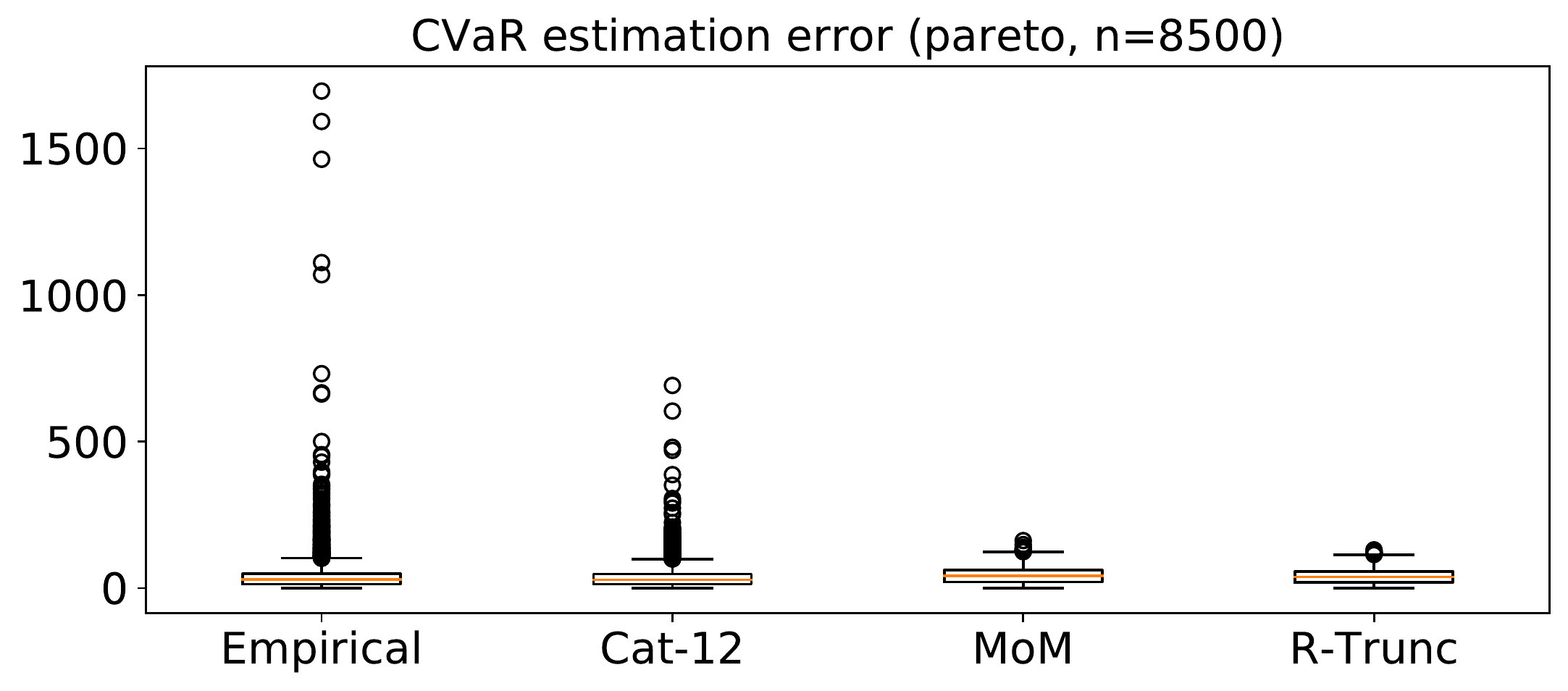}
\caption{Analysis of deviations over $n$, for fixed $\alpha = 0.05$. Top: folded-Normal. Middle: log-Normal. Bottom: Pareto.}
\label{fig:static_n}
\end{figure}

\begin{figure}[t]
\centering
\includegraphics[width=0.5\textwidth]{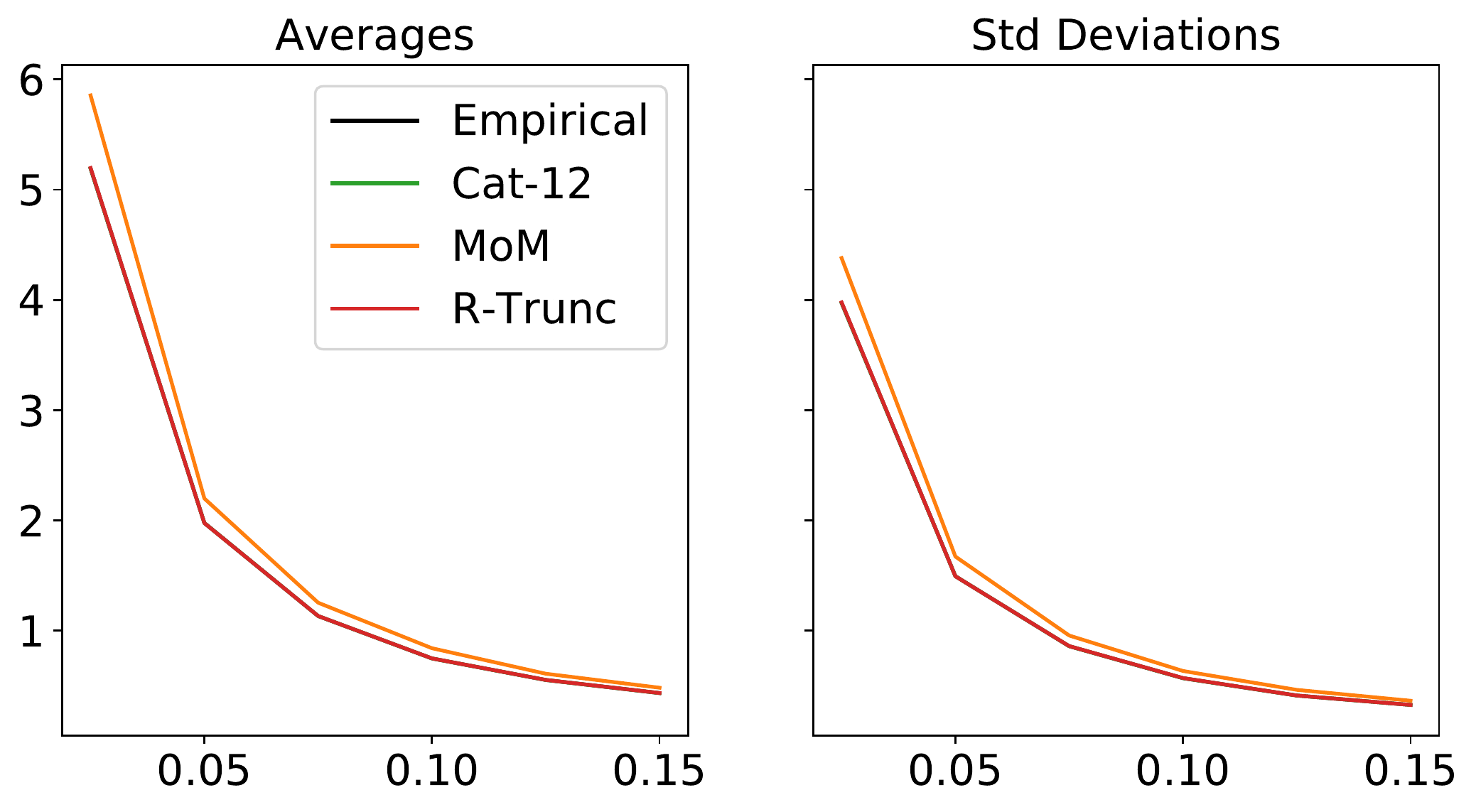}\,\includegraphics[width=0.5\textwidth]{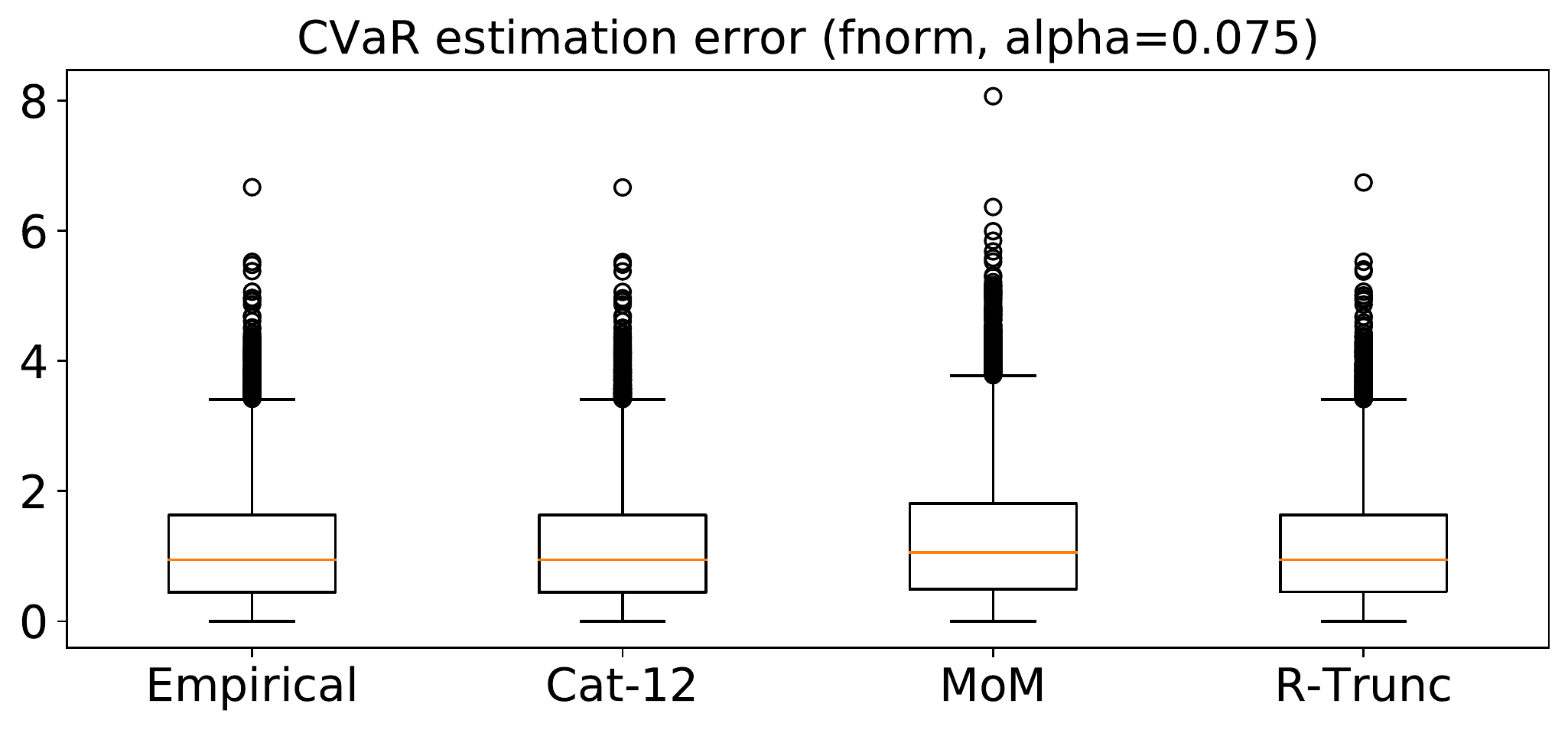}\\
\includegraphics[width=0.5\textwidth]{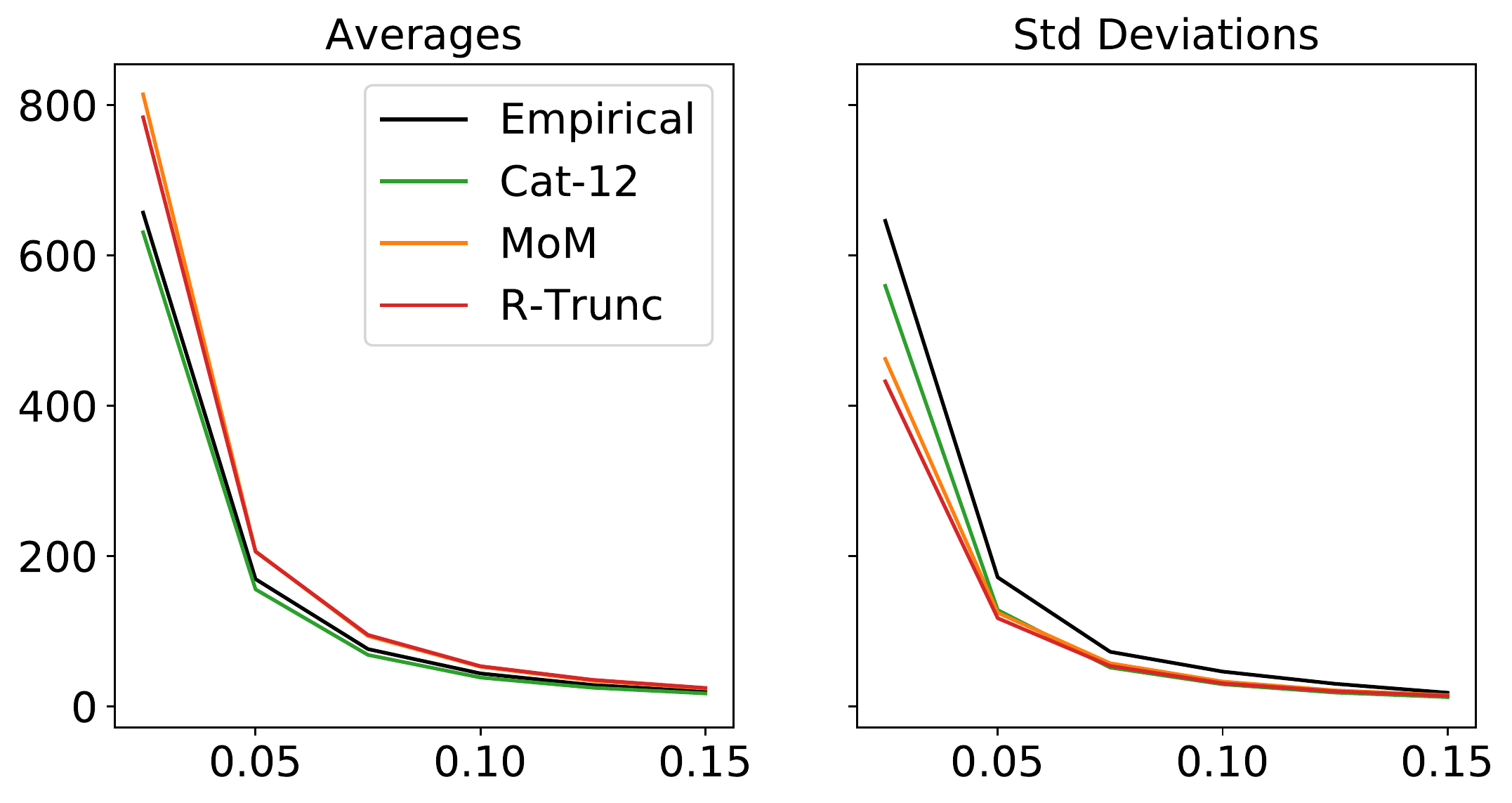}\,\includegraphics[width=0.5\textwidth]{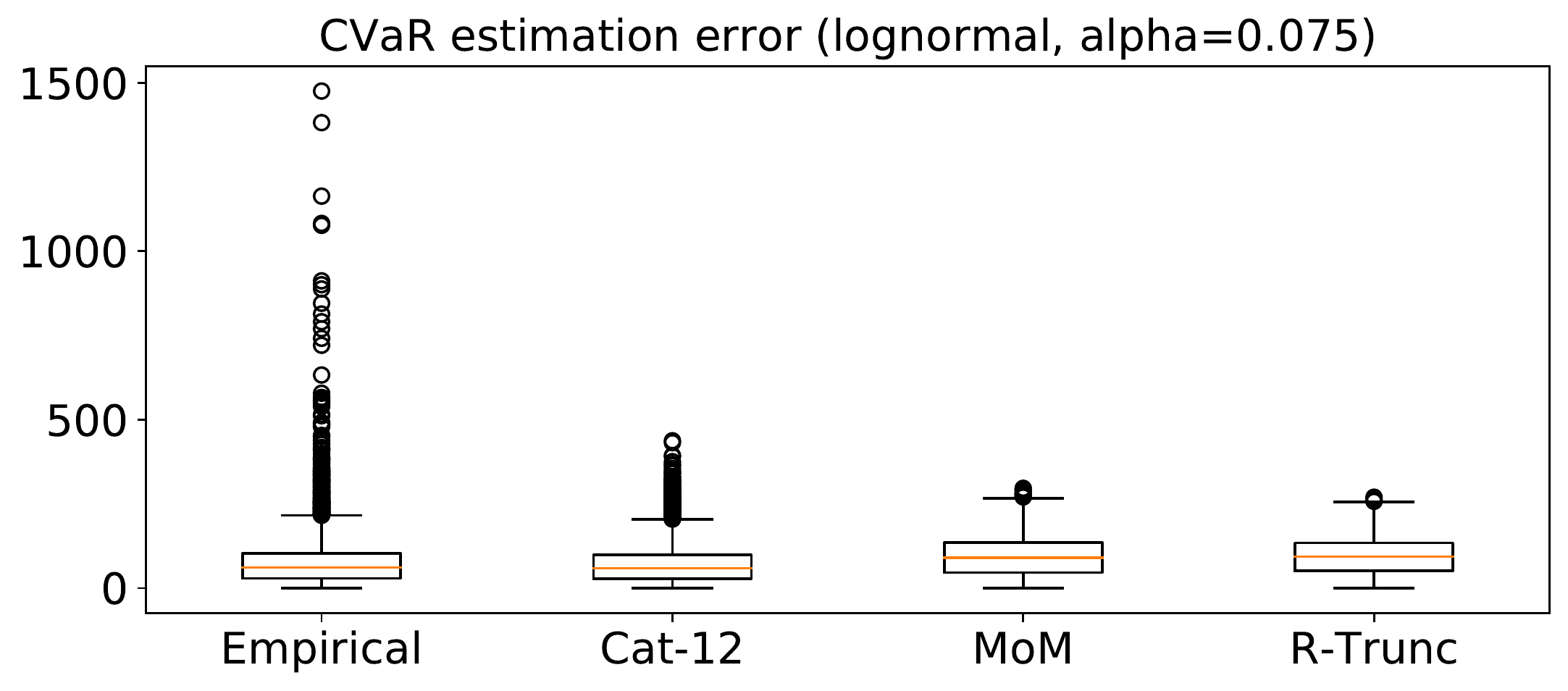}\\
\includegraphics[width=0.5\textwidth]{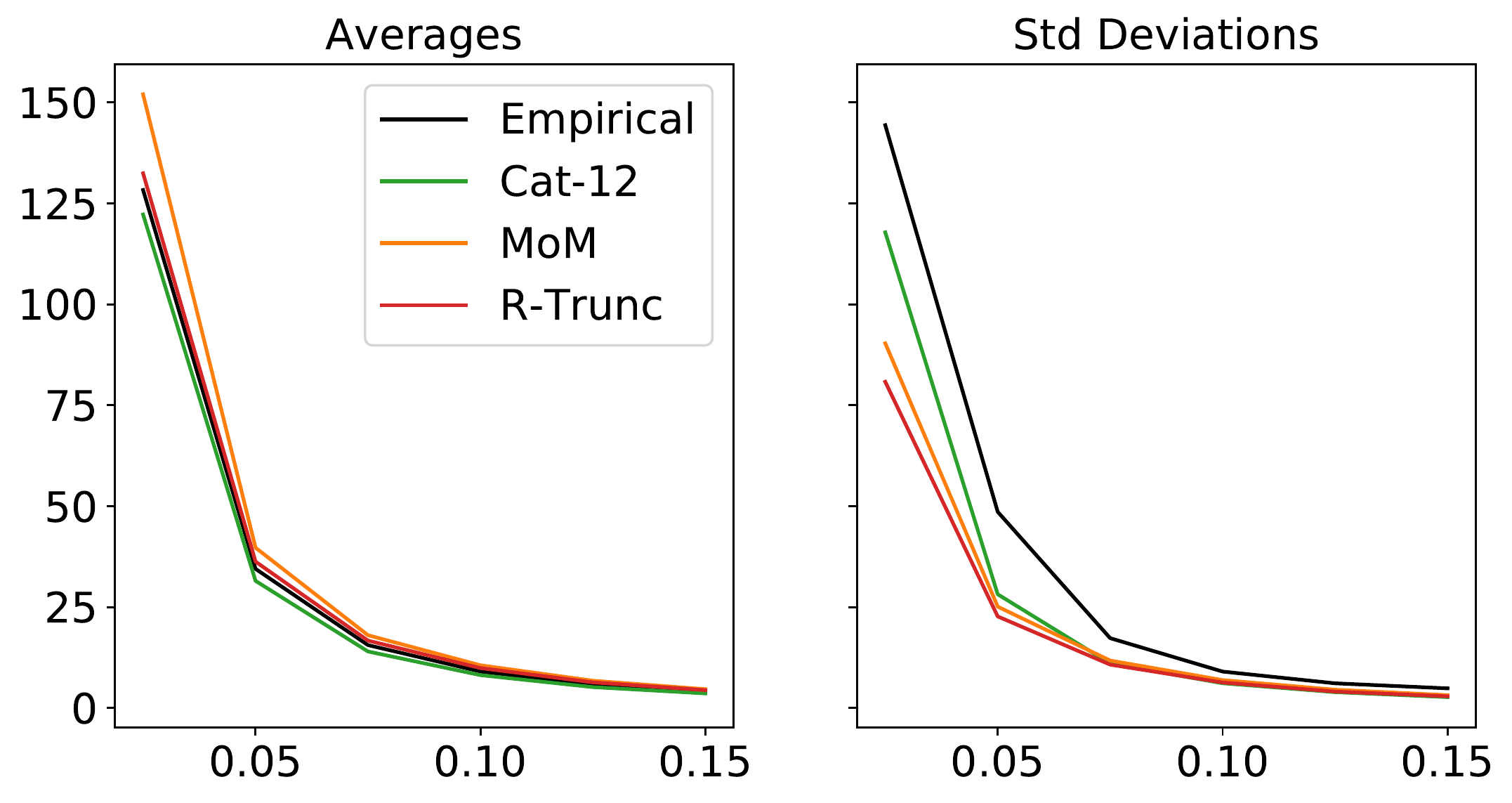}\,\includegraphics[width=0.5\textwidth]{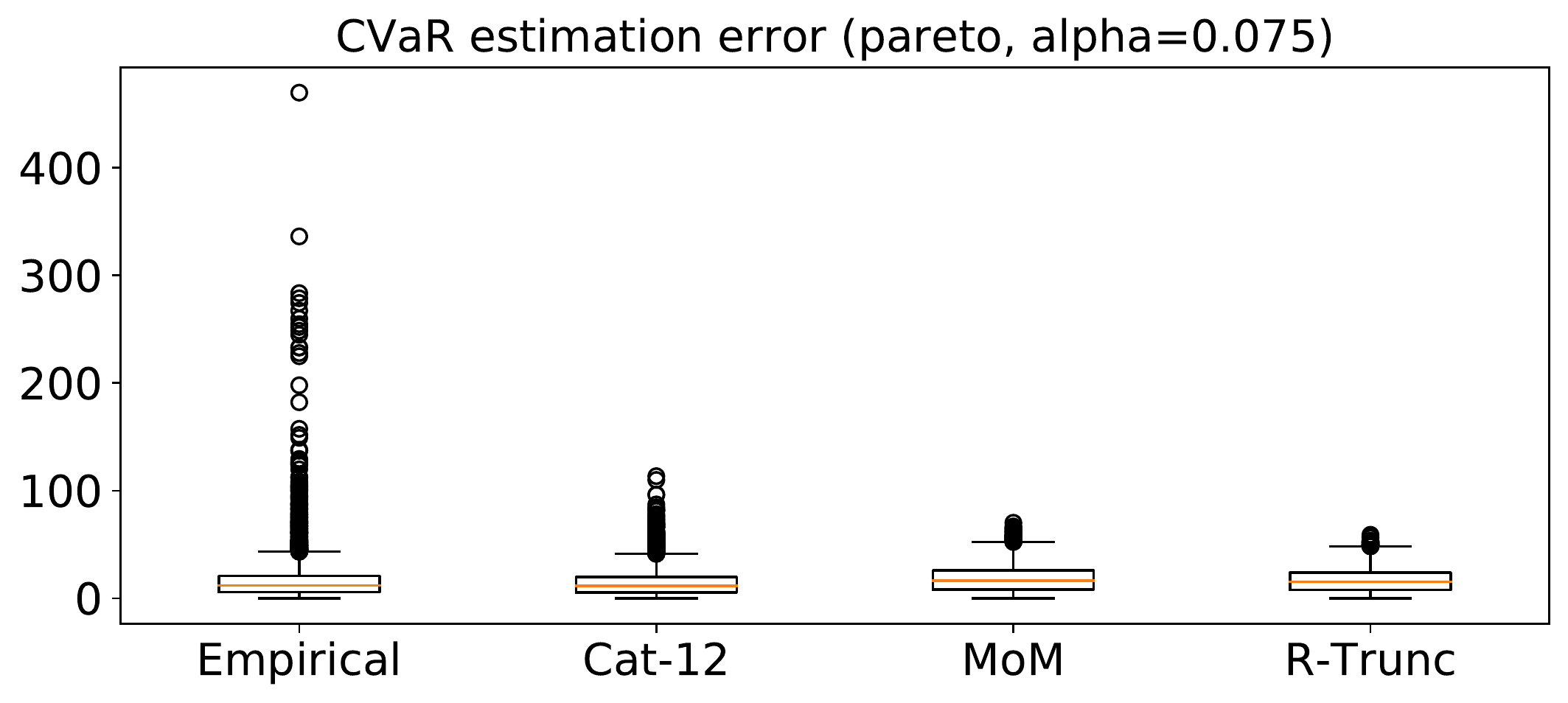}
\caption{Analysis of deviations over $\alpha$, for fixed $n = 10000$. Top: folded-Normal. Middle: log-Normal. Bottom: Pareto.}
\label{fig:static_alpha}
\end{figure}

\subsection{Application to learning algorithms}

\paragraph{Experimental setup}

As a natural first application, we consider linear regression in the context of CVaR-based learning. That is, random data are generated as pairs $Z = (X,Y) \sim \ddist$ following the relation $Y = \langle \wstar, X\rangle + E$, where $E$ is a zero-mean random noise term independent of $X$, and $\wstar \in \WW$ is some pre-fixed vector. We consider two types of losses, namely squared error and absolute deviations, respectively amounting to $\loss(w;Z) = (\langle w-\wstar,X\rangle - E)^{2}/2$ and $\loss(w;Z) = |\langle w-\wstar,X\rangle - E|$. The learner does not know $\wstar$ and cannot observe $E$ directly, all it has is access to $X$ and $Y$, and thus the final loss values (and resulting partial derivatives, etc.). The main reason for studying two different losses is as follows. The squared error is used very commonly in practice, but does not satisfy the $\parasm_{\loss}$-Lipschitz requirement made by $\text{\ref{asmp:data}}$ unless the noise $E$ is bounded. In contrast, the absolute error satisfies the Lipschitz requirement even when $E$ is unbounded and heavy-tailed. One point of interest will be to compare these two cases, and see how far the theoretical insights from Theorem \ref{thm:error_bd_dynamic} extend beyond the formal conditions.

Regarding the methods to be studied, we compare Algorithm \ref{algo:dynamic} (denoted \texttt{RV-SGDAve}) with three well-known baseline methods. As a classical baseline, we consider a batch gradient descent implementation of empirical CVaR risk minimization (denoted \texttt{ERM-GD}), i.e., typical iterative gradient descent where the update direction comes from the gradient (or sub-gradient) of the usual empirical estimate of $F_{\alpha}(w,v)$. Note that this is an update in $d+1$ dimensions optimizing both $w \in \WW$ and $v \in \RR$, so no direct estimates of $\vv_{\alpha}$ are made. We consider two alternative learning algorithms, which were designed (in the context of \emph{risk} estimation) to be robust and computationally efficient under potentially heavy-tailed losses. These are robust gradient descent routines based on M-estimation \citep{holland2019c} and median-of-means \citep{chen2017a,prasad2018a}, respectively denoted \texttt{RGD-M} and \texttt{RGD-MoM}. Essentially, instead of simply taking the empirical means of the sampled sub-gradients of $f(w,v;Z)$ as is done by \texttt{ERM-GD}, these \texttt{RGD-*} methods incorporate an extra sub-routine at each step for aggregating the sub-gradients in a robust way such that the impact of outliers is dampened, reducing superfluous random exploration in a convex loss setting.

We study the impact that changes in the underlying distribution $\ddist$ have on different learning algorithms at fixed levels of $n$, $d$, and $\alpha$. For simplicity, in the nascent tests that we have conducted here, we fix $n = 500$, $d=2$, and $\alpha = 0.05$ throughout. In all experiments, $X$ follows an isotropic standard multivariate Normal distribution, and it is the distribution of additive noise $E$ that we control as a key experimental condition. Fixing $A \sim \text{Normal}(0,b^{2})$, we consider $E = A - \exx A$ (Normal case), $E = \ct{e}^{A} - \exx\ct{e}^{A}$ (log-Normal case), and finally $E = A^{\prime} - \exx A^{\prime}$ where $A^{\prime}$ has a Pareto distribution (Pareto case). To control the signal/noise ratio, we set the parameters such that all three cases, the width of the interquartile range of $E$ is constant, at a value of $3.0$. More precisely, we set $b = 2.2$ for the Normal case, $b = 1.75$ for the log-Normal case, and set $A^{\prime}$ to have a Pareto distribution with shape $2.1$ and scale $3.5$.\footnote{This noise is generated using the Python library \texttt{scipy} (ver.~1.4.1), in particular via the function \texttt{scipy.stats.pareto(b,scale)}, where we have $\texttt{b}=2.1$ and $\texttt{scale}=3.5$.} Batch methods are set to have a fixed step size of $0.1 / \sqrt{d}$, while Algorithm \ref{algo:dynamic} has a fixed step size of $0.01 / \sqrt{d}$. All methods are run until they spend a fixed ``budget,'' where the cost is measured in terms of gradient evaluations, i.e., one cost is spent each time a sub-gradient of $f(w,v;Z_{i})$ is computed for any $(w,v)$ and any $i$. The budget for all methods is fixed to $40n$; this means Algorithm \ref{algo:dynamic} is allowed to take multiple passes over the data, going beyond the scope of Theorem \ref{thm:error_bd_dynamic}; the stability beyond the single-pass threshold is a natural point to study empirically. We note that all numerical experiments have been implemented using Python (ver.~3.8), using just libraries \texttt{numpy} (ver.~1.18) and \texttt{scipy} (ver.~1.4.1).\footnote{An online repository of code to re-create the experiments here will be made available soon.}

\paragraph{Results and discussion}

Our main results for this section are summarized in Figure \ref{fig:POC_CVaR}. Here ``excess CVaR risk'' refers to $F_{\alpha}(w,v)-F_{\alpha}^{\ast}$ approximated on an independent large test set of size $10^{5}$, where $F_{\alpha}^{\ast}$ is set to the value achieved by an oracle batch gradient descent routine using the full test run for many iterations. Thus the performance is relative, stated with respect to what could be achieved given a sample many orders of magnitude larger. We have run $250$ independent trials of this experiment, and the average and standard deviation values in Figure \ref{fig:POC_CVaR} reflect statistics taken over these trials. The immediate take-away is that the proposed algorithm offers an appealing improvement in efficiency, realizing superior CVaR-risk using far less operations. Furthermore, this is robust both to the underlying distribution, and the nature of the underlying loss. That is, even when the $\parasm_{\loss}$-Lipschitz assumption on the loss breaks down (left-hand side of Figure \ref{fig:POC_CVaR}), we see competitive behaviour.

\begin{figure}[t]
\centering
\includegraphics[width=0.5\textwidth]{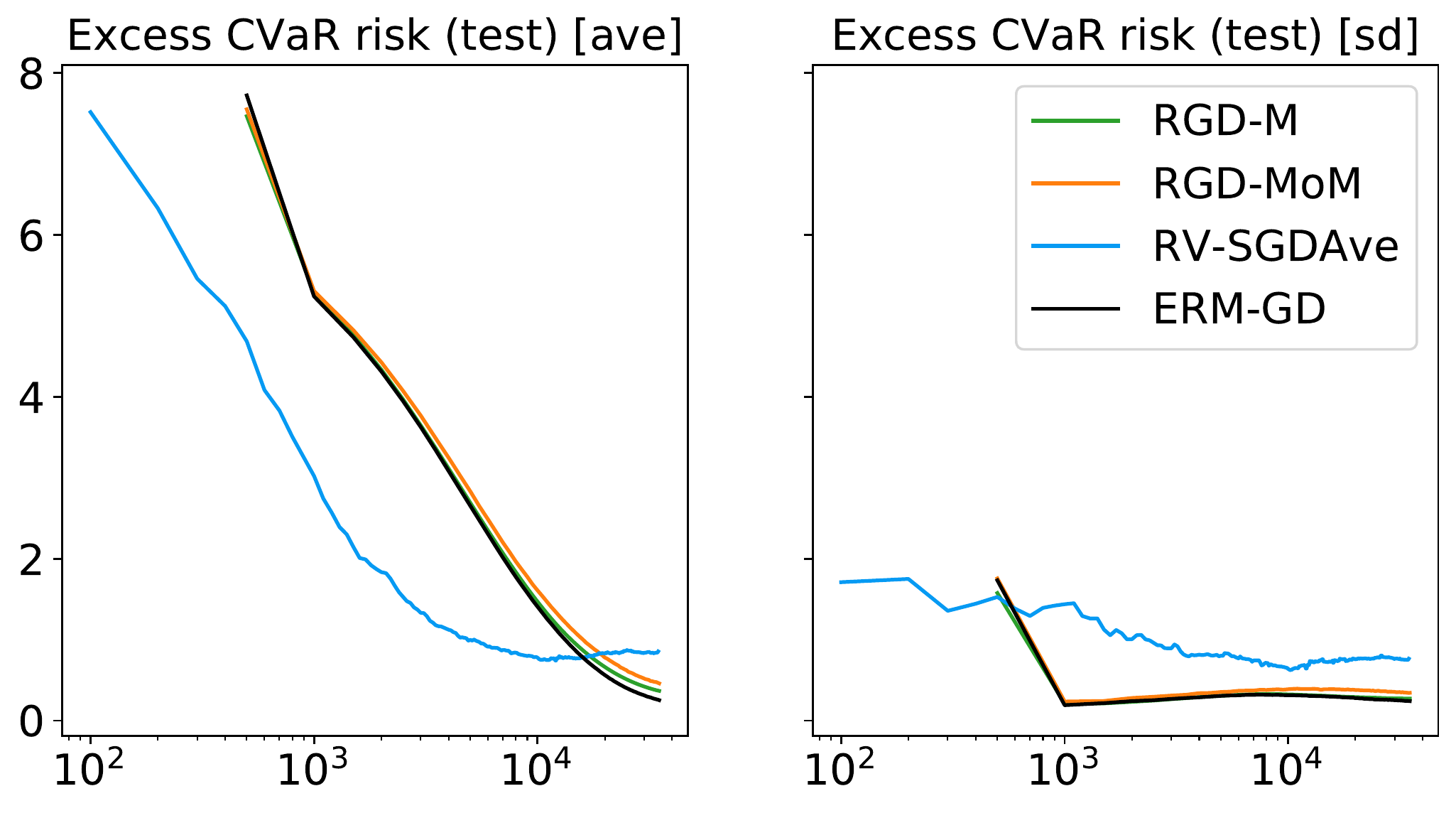}\,\includegraphics[width=0.5\textwidth]{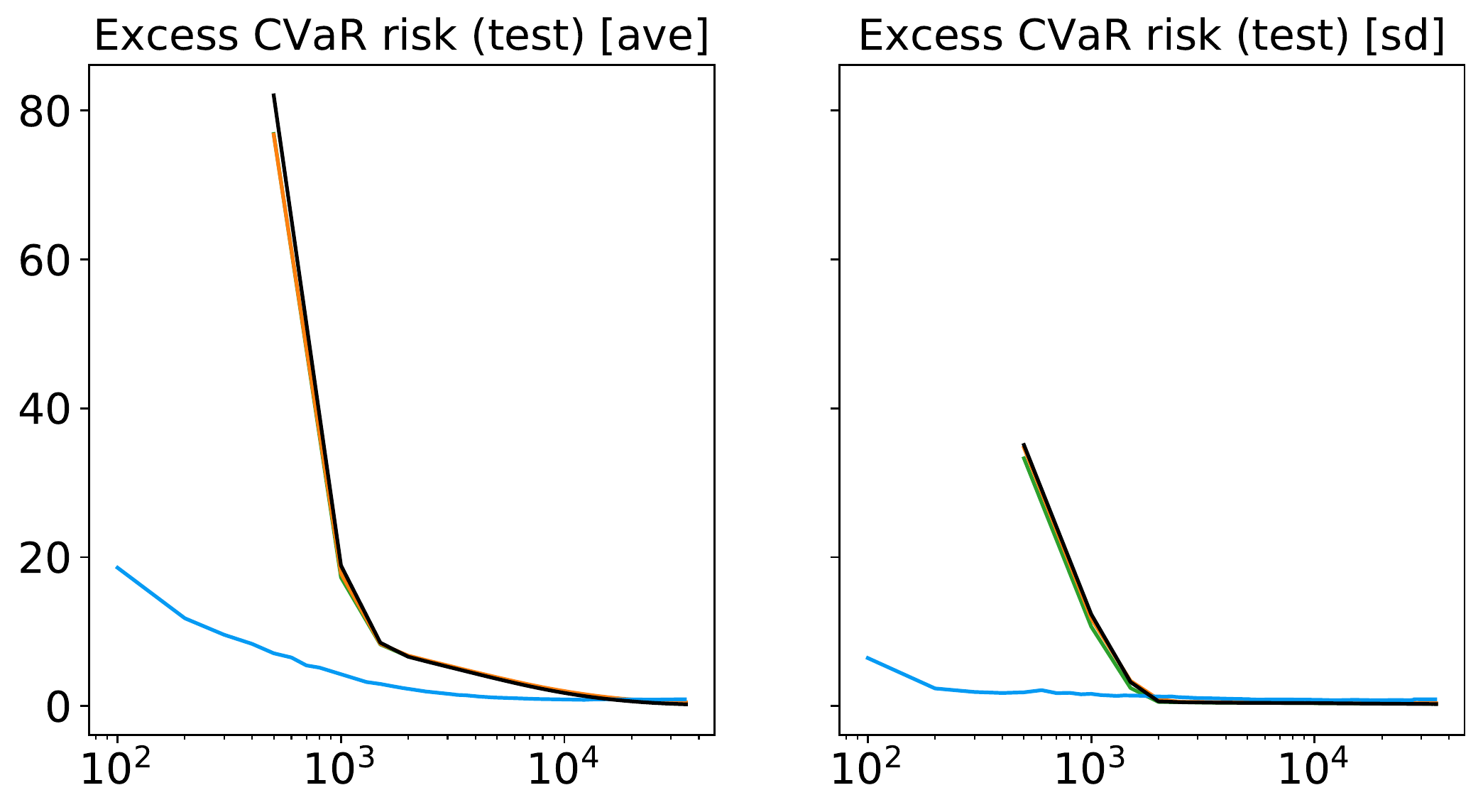}\\
\includegraphics[width=0.5\textwidth]{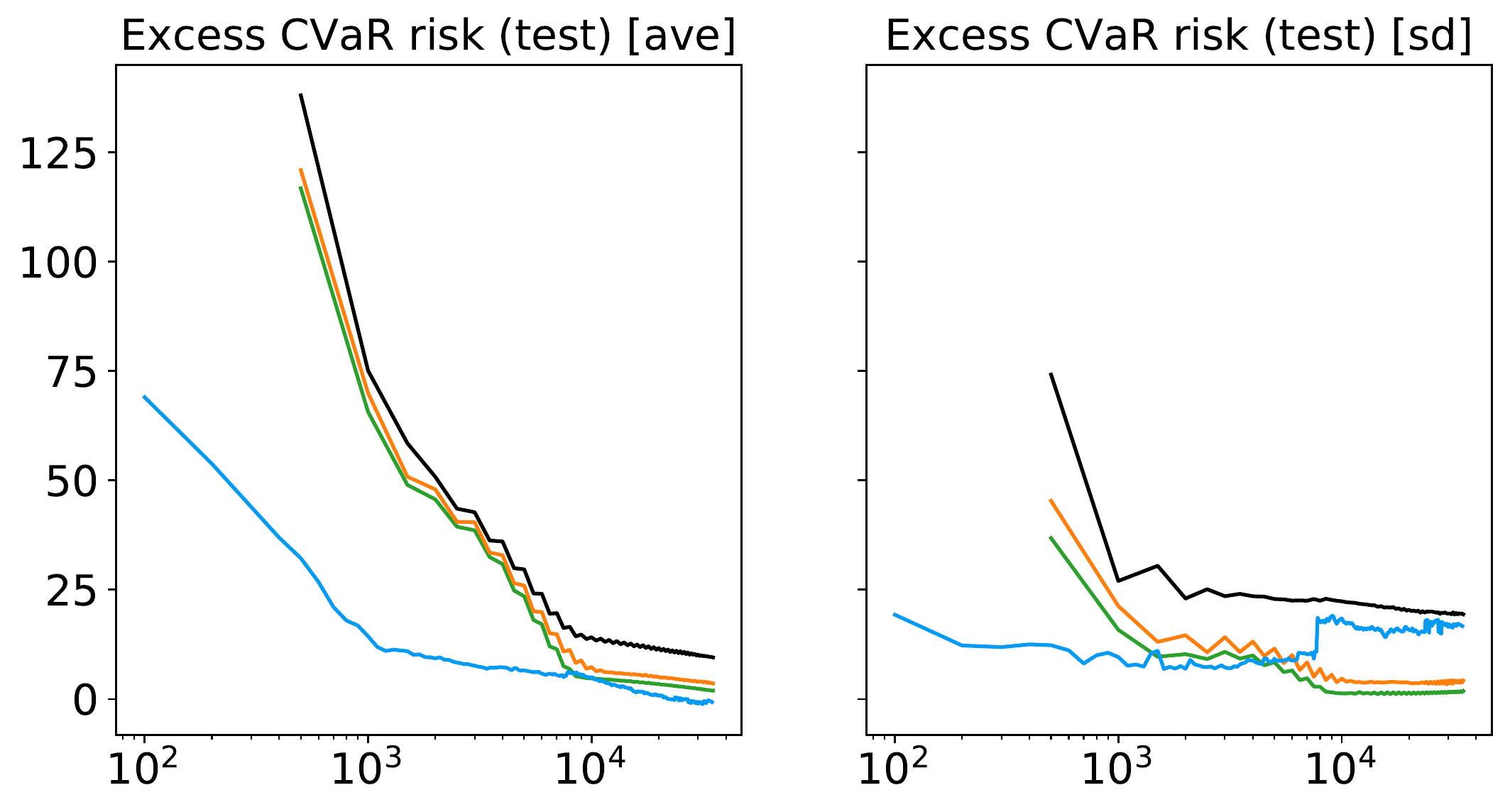}\,\includegraphics[width=0.5\textwidth]{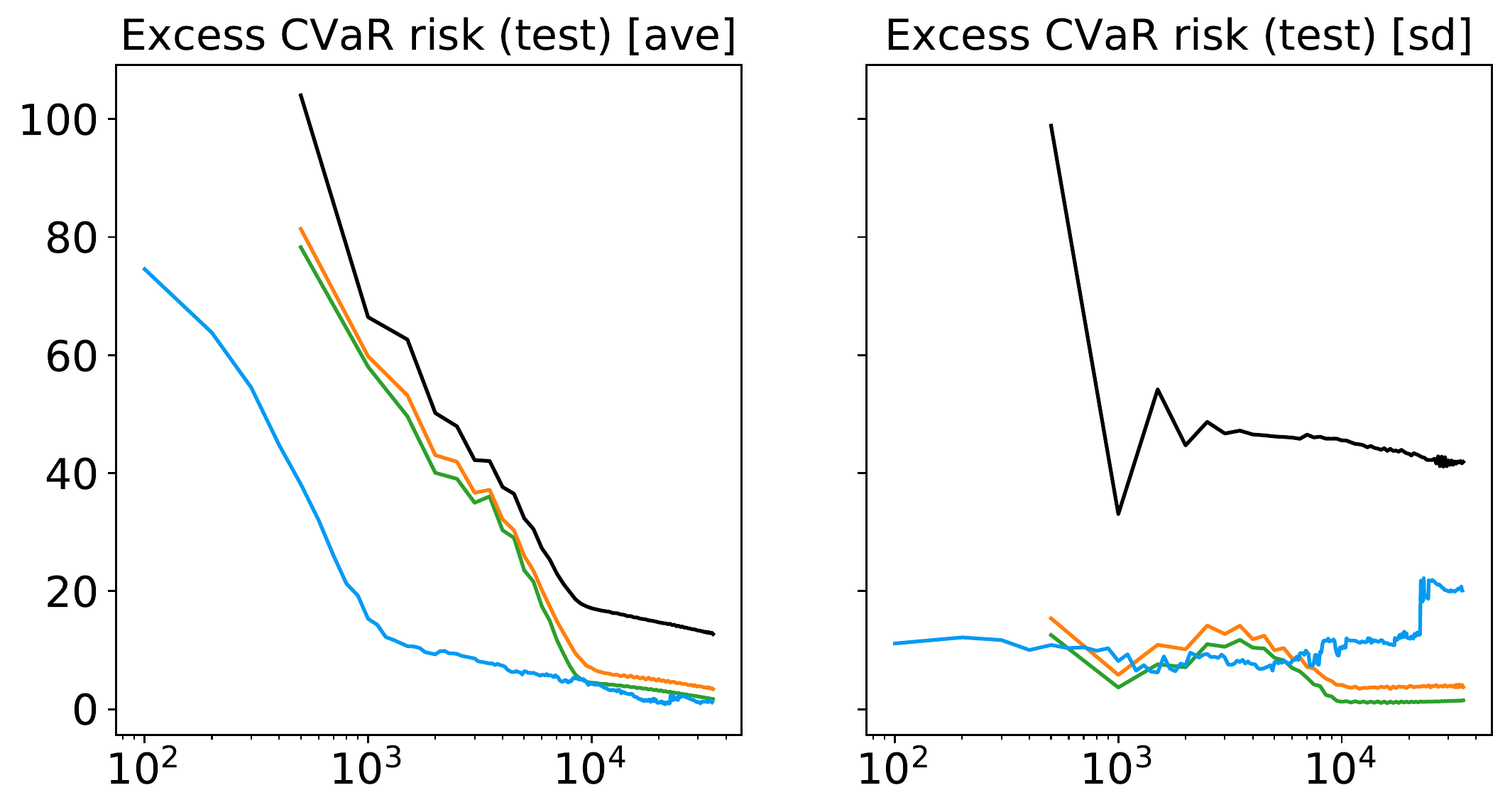}\\
\includegraphics[width=0.5\textwidth]{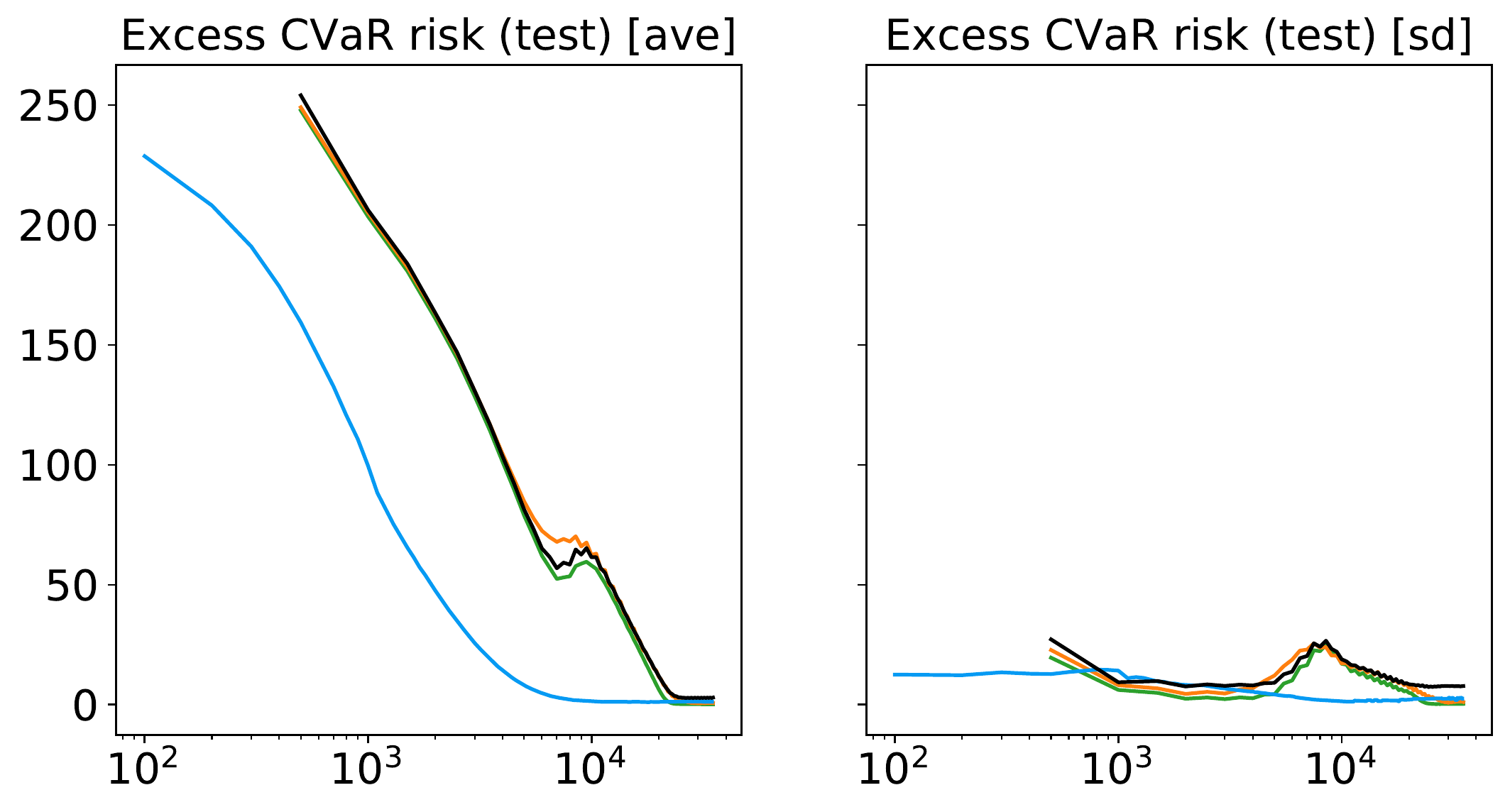}\,\includegraphics[width=0.5\textwidth]{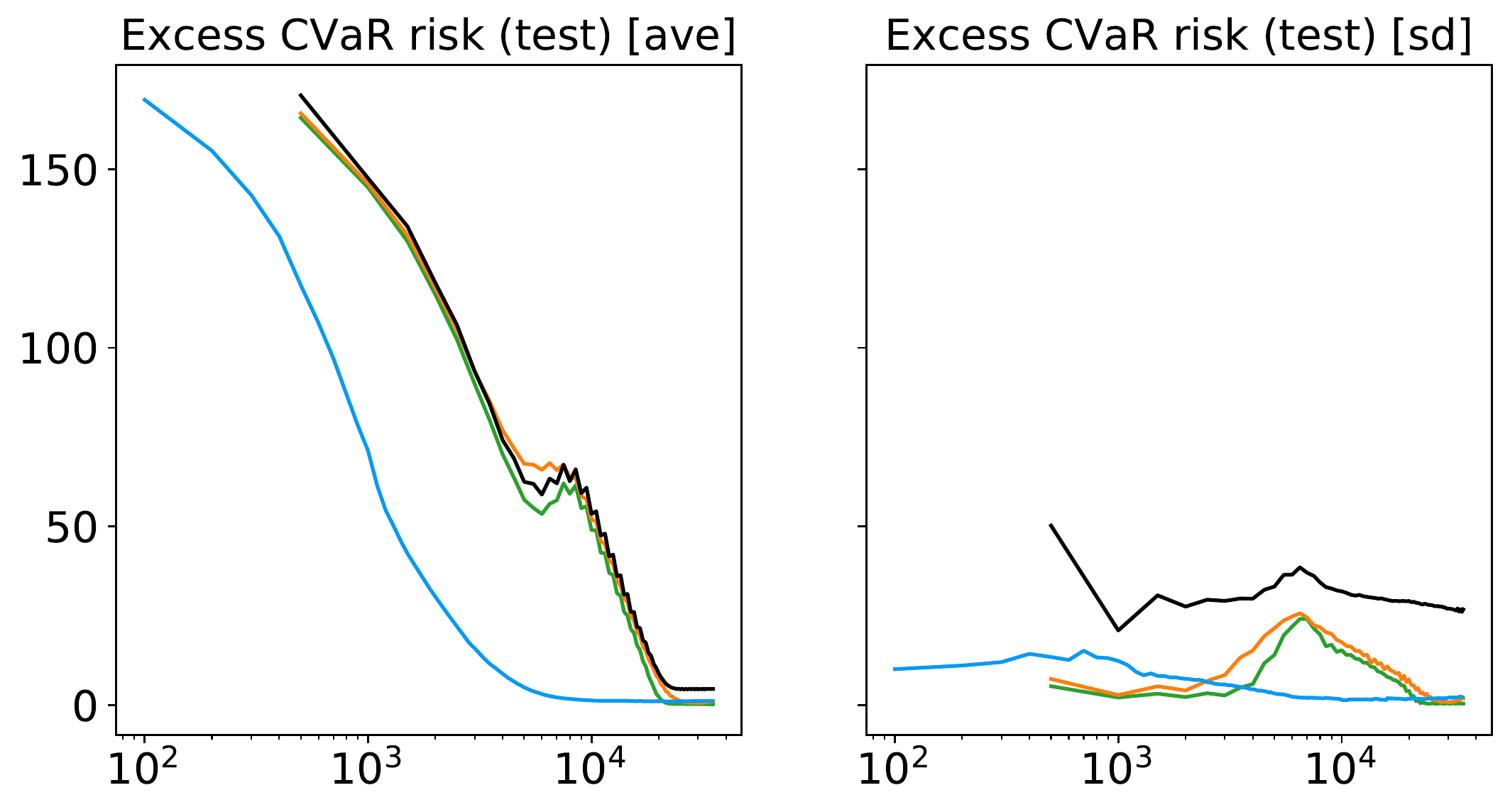}
\caption{Excess CVaR risk for squared error (left-most plots) and absolute error (right-most plots). Top: folded-Normal. Middle: log-Normal. Bottom: Pareto.}
\label{fig:POC_CVaR}
\end{figure}

\section{Future directions}

There are several interesting lines of work that can be taken up based on the initial results presented here. One direction is to go beyond CVaR to more diverse classes of metrics/feedback. One obvious approach is to consider general coherent risk metrics under potentially heavy-tailed data. Another is to try and extend the analysis to obtain results for completely distinct performance classes that in some sense mimic human loss/reward systems (e.g., cumulative prospect theory). Initial explorations have been made by \citet{bhat2020a}, but the basic theory and algorithmic analysis are still far from complete. Other notions of conditional expectation, which do not necessarily depend on quantiles, is another natural approach of interest. An alternative direction of interest is to deepen and expand upon the empirical studies we have started here, looking at large-scale learning problems and real-world data sets and models, potentially without convexity. Virtually all work done on CVaR-driven learning algorithms has made use of Lipschitz losses, but this precludes the possibility of heavy-tailed \emph{gradients}. Developing new procedures and analytical machinery to tackle this setting is another avenue that can be considered quite promising.

\appendix

\section{Technical appendix}

\begin{proof}[Proof of Lemma \ref{lem:vhat_nice_order}]
Results of this nature are well-known, but we give a proof for completeness. Starting with the left-most inequality, say $Y^{\ast}_{(1-\alpha)n}$ is \emph{less} than $\vv_{2\alpha}$. This means that at least $(1-\alpha)n$ points from $\Y_{n}$ were below $\vv_{2\alpha}$, or in terms of the empirical CDF, that $\widehat{F}_{n}(\vv_{2\alpha}) > 1-\alpha$. Note that $n\widehat{F}_{n}(\vv_{2\alpha}) \sim B(n,p)$, a binomial random variable with $p=1-2\alpha$. Using this connection, we have
\begin{align*}
\prr\left\{ \widehat{F}_{n}(\vv_{2\alpha}) > 1-\alpha \right\} = \prr\left\{ \frac{B(n,p)}{n}-p > \alpha \right\} \leq \exp\left(-\frac{n\alpha^{2}}{2p(1-p)}\right) \leq \exp\left(-\frac{n\alpha}{4}\right),
\end{align*}
where the exponential tail bound dates back to \citet[Thm.~2]{okamoto1959a}. It thus follows that we have $\prr\{\vv_{2\alpha} \leq Y_{(1-\alpha)n}^{\ast} \} \geq 1-\exp(-n\alpha/4)$.

For the upper bound, in a perfectly analogous fashion, the bad event where $Y_{(1-\alpha)n}^{\ast}$ exceeds $\vv_{\alpha/2}$ is equivalent to $\{ B(n,p^{\prime}) > n\alpha \}$, where $p^{\prime}=\alpha/2$. The bounds of \citet{okamoto1959a} in this case do not provide the desired dependence on $\alpha$, so a direct application of Bernstein's inequality (one-sided) for bounded random variables will instead be used \citep[Ch.~2]{boucheron2013a}. Using a 
\begin{align*}
\prr\left\{ Y_{(1-\alpha)n}^{\ast} > \vv_{\alpha/2} \right\} = \prr\left\{ \frac{B(n,p^{\prime})}{n} > \alpha \right\} = \prr\left\{ \frac{B(n,p^{\prime})}{n} - p^{\prime} > \frac{\alpha}{2} \right\} \leq \exp\left(-\frac{3n\alpha}{14} \right).
\end{align*}
The desired result follows immediately from a union bound over the two bad events, using the looser of the two bounds.
\end{proof}

\begin{lem}\label{lem:lipschitz_dual_bound_characterization}
Let $f: \VV \to \RR$ be convex. Then, $f$ is $\parasm$-Lipschitz with respect to norm $\|\cdot\|$ if and only if $\|u\|_{\star} \leq \parasm$ for all $u \in \partial f(v)$ and $v \in \VV$.
\end{lem}
\begin{proof}
See \citet[Lem.~2.6]{shalev2012a} for a proof.
\end{proof}

\begin{proof}[Proof of Lemma \ref{lem:learn_conv_lip_SGDave}]
This result follows from direct application of well-known SGD analysis, for example \citet[Sec.~2.2]{nemirovski2009a} or \citet[Sec.~14.5.1]{shalev2014a}, and simply requires that the sub-gradients used are unbiased estimates of some sub-gradient of $F_{\alpha}$, namely that in (\ref{eqn:sgd_defn}) the update directions satisfy $\exx_{\ddist} G_{\alpha}(w,v;Z) \in \partial F_{\alpha}(w,v)$. Fortunately, convexity of $f_{\alpha}$ implies that $\partial F_{\alpha}(w,v) = \{ \exx_{\ddist}G: G \in \partial f_{\alpha}(w,v;Z) \}$ holds \citep{strassen1965a,nemirovski2009a}, meaning that the assumptions of the cited works are satisfied.
\end{proof}
\begin{proof}[Proof of Lemma \ref{lem:dynamic_lip}]
First of all, the convexity of $(w,v) \mapsto f_{\alpha}(w,v;z)$ follows from the convexity of $w \mapsto \loss(w;z)$, and elementary calculus of convex functions, e.g.~\citet[Thm.~5.1]{rockafellar1970a}. Next, note that the sub-differential of $f_{\alpha}(w,v;z)$ takes the form\footnote{See \citet[Ch.~3]{bertsekas2015ConvexOpt} for a general reference, or \citet[Sec.~4]{rockafellar2000a} for the CVaR case.}
\begin{align*}
\partial f_{\alpha}(w,v;z) =
\begin{cases}
\left\{ \frac{1}{\alpha}\left(\nabla\loss(w;z), \alpha-1\right) \right\}, & \text{ if } \loss(w;z) > v\\
\left\{ \frac{1}{\alpha}\left(a\nabla\loss(w;z), \alpha-a\right): a \in [0,1] \right\}, & \text{ if } \loss(w;z) = v\\
\left\{ \left(\mv{0}, 1\right) \right\}, & \text{ if } \loss(w;z) < v.
\end{cases}
\end{align*}
Since $\loss(\cdot;z)$ is convex, it follows from Lemma \ref{lem:lipschitz_dual_bound_characterization} that for any $g \in \partial \loss(w;z)$, we have $\|g\| \leq \parasm$, regardless of choice of $w$ or $z$. Since we are assuming $\loss(\cdot;z)$ is differentiable, the sub-differential contains only the gradient $\partial \loss(w;z)=\{\nabla \loss(w;z)\}$, and we thus have $\|\nabla \loss(w;z)\| \leq \parasm$. Applying this to each vector in $\partial f_{\alpha}(w,v;z)$ given above, we clearly have
\begin{align*}
g \in \partial f_{\alpha}(w,v;z) \text{ satisfies }
\begin{cases}
\|g\| \leq \frac{\sqrt{\parasm^{2} + (1-\alpha)^{2}}}{\alpha}, & \text{ if } \loss(w;z) \geq v\\
\|g\| \leq 1, & \text{ if } \loss(w;z) < v.
\end{cases}
\end{align*}
Since these $\ell_{2}$ norm bounds hold for any choice of $w$, $v$, and $z$, then applying Lemma \ref{lem:lipschitz_dual_bound_characterization} once again, it follows that $f_{\alpha}(w,v;z)$ is $\parasm_{\alpha}$-Lipschitz continuous, where $\parasm_{\alpha}$ is as defined in the lemma statement. Finally, note that the convexity and $\parasm_{\alpha}$-Lipschitz continuity of the map $(w,v) \mapsto F_{\alpha}(w,v)$ follows immediately from the stronger properties just shown for $f_{\alpha}$.
\end{proof}

\bibliographystyle{apalike}
\bibliography{../refs/refs}

\begin{thebibliography}{}

\bibitem[Anthony and Bartlett, 1999]{anthony1999NNTheory}
Anthony, M. and Bartlett, P.~L. (1999).
\newblock {\em Neural Network Learning: Theoretical Foundations}.
\newblock Cambridge University Press.

\bibitem[Artzner et~al., 1999]{artzner1999a}
Artzner, P., Delbaen, F., Eber, J.-M., and Heath, D. (1999).
\newblock Coherent measures of risk.
\newblock {\em Mathematical Finance}, 9(3):203--228.

\bibitem[Bertsekas, 2015]{bertsekas2015ConvexOpt}
Bertsekas, D.~P. (2015).
\newblock {\em Convex Optimization Algorithms}.
\newblock Athena Scientific.

\bibitem[Bhat and Prashanth, 2020]{bhat2020a}
Bhat, S.~P. and Prashanth, L.~A. (2020).
\newblock Concentration of risk measures: A {W}asserstein distance approach.
\newblock In {\em Advances in Neural Information Processing Systems 32 (NeurIPS
  2019)}.

\bibitem[Boucheron et~al., 2013]{boucheron2013a}
Boucheron, S., Lugosi, G., and Massart, P. (2013).
\newblock {\em Concentration inequalities: a nonasymptotic theory of
  independence}.
\newblock Oxford University Press.

\bibitem[Brownlees et~al., 2015]{brownlees2015a}
Brownlees, C., Joly, E., and Lugosi, G. (2015).
\newblock Empirical risk minimization for heavy-tailed losses.
\newblock {\em Annals of Statistics}, 43(6):2507--2536.

\bibitem[Bubeck et~al., 2013]{bubeck2013a}
Bubeck, S., Cesa-Bianchi, N., and Lugosi, G. (2013).
\newblock Bandits with heavy tail.
\newblock {\em IEEE Transactions on Information Theory}, 59(11):7711--7717.

\bibitem[Cardoso and Xu, 2019]{cardoso2019a}
Cardoso, A.~R. and Xu, H. (2019).
\newblock Risk-averse stochastic convex bandit.
\newblock In {\em 22nd International Conference on Artificial Intelligence and
  Statistics (AISTATS)}, volume~89 of {\em Proceedings of Machine Learning
  Research}, pages 39--47.

\bibitem[Catoni, 2012]{catoni2012a}
Catoni, O. (2012).
\newblock Challenging the empirical mean and empirical variance: a deviation
  study.
\newblock {\em Annales de l'Institut Henri Poincar{\'e}, Probabilit{\'e}s et
  Statistiques}, 48(4):1148--1185.

\bibitem[Chen et~al., 2017]{chen2017a}
Chen, Y., Su, L., and Xu, J. (2017).
\newblock Distributed statistical machine learning in adversarial settings:
  {B}yzantine gradient descent.
\newblock In {\em Proceedings of the ACM on Measurement and Analysis of
  Computing Systems}, volume~1. ACM.

\bibitem[Chow et~al., 2016]{chow2016a}
Chow, Y., Tamar, A., Mannor, S., and Pavone, M. (2016).
\newblock Risk-sensitive and robust decision-making: a {CVaR} optimization
  approach.
\newblock In {\em Advances in Neural Information Processing Systems 28 (NIPS
  2015)}, pages 1522--1530.

\bibitem[Devroye et~al., 1996]{devroye1996ProbPR}
Devroye, L., Gy{\"o}rfi, L., and Lugosi, G. (1996).
\newblock {\em A Probabilistic Theory of Pattern Recognition}.
\newblock Springer.

\bibitem[Devroye et~al., 2016]{devroye2016a}
Devroye, L., Lerasle, M., Lugosi, G., and Oliveira, R.~I. (2016).
\newblock Sub-gaussian mean estimators.
\newblock {\em Annals of Statistics}, 44(6):2695--2725.

\bibitem[Haussler, 1992]{haussler1992a}
Haussler, D. (1992).
\newblock Decision theoretic generalizations of the {PAC} model for neural net
  and other learning applications.
\newblock {\em Information and Computation}, 100(1):78--150.

\bibitem[Holland, 2020a]{holland2020nonsc}
Holland, M.~J. (2020a).
\newblock Improved scalability under heavy tails, without strong convexity.
\newblock {\em arXiv preprint arXiv:2006.01364v1}.

\bibitem[Holland, 2020b]{holland2020a}
Holland, M.~J. (2020b).
\newblock {PAC}-{B}ayes under potentially heavy tails.
\newblock In {\em Advances in Neural Information Processing Systems 32 (NeurIPS
  2019)}.

\bibitem[Holland and Ikeda, 2019]{holland2019c}
Holland, M.~J. and Ikeda, K. (2019).
\newblock Better generalization with less data using robust gradient descent.
\newblock In {\em 36th International Conference on Machine Learning (ICML)},
  volume~97 of {\em Proceedings of Machine Learning Research}.

\bibitem[Hsu and Sabato, 2016]{hsu2016a}
Hsu, D. and Sabato, S. (2016).
\newblock Loss minimization and parameter estimation with heavy tails.
\newblock {\em Journal of Machine Learning Research}, 17(18):1--40.

\bibitem[Kagrecha et~al., 2020]{kagrecha2020a}
Kagrecha, A., Nair, J., and Jagannathan, K. (2020).
\newblock Distribution oblivious, risk-aware algorithms for multi-armed bandits
  with unbounded rewards.
\newblock In {\em Advances in Neural Information Processing Systems 32 (NeurIPS
  2019)}.

\bibitem[Kolla et~al., 2019]{kolla2019a}
Kolla, R.~K., Prashanth, L.~A., Bhat, S.~P., and Jagannathan, K. (2019).
\newblock Concentration bounds for empirical conditional value-at-risk: The
  unbounded case.
\newblock {\em Operations Research Letters}, 47(1):16--20.

\bibitem[Kosorok, 2008]{kosorok2008EPSP}
Kosorok, M.~R. (2008).
\newblock {\em Introduction to Empirical Processes and Semiparametric
  Inference}.
\newblock Springer.

\bibitem[Krokhmal et~al., 2002]{krokhmal2002a}
Krokhmal, P., Palmquist, J., and Uryasev, S. (2002).
\newblock Portfolio optimization with conditional value-at-risk objective and
  constraints.
\newblock {\em Journal of Risk}, 4:43--68.

\bibitem[Lerasle and Oliveira, 2011]{lerasle2011a}
Lerasle, M. and Oliveira, R.~I. (2011).
\newblock Robust empirical mean estimators.
\newblock {\em arXiv preprint arXiv:1112.3914}.

\bibitem[Lugosi and Mendelson, 2019a]{lugosi2019b}
Lugosi, G. and Mendelson, S. (2019a).
\newblock Mean estimation and regression under heavy-tailed distributions: A
  survey.
\newblock {\em Foundations of Computational Mathematics}, 19(5):1145--1190.

\bibitem[Lugosi and Mendelson, 2019b]{lugosi2019a}
Lugosi, G. and Mendelson, S. (2019b).
\newblock Robust multivariate mean estimation: the optimality of trimmed mean.
\newblock {\em arXiv preprint arXiv:1907.11391v1}.

\bibitem[Mansini et~al., 2007]{mansini2007a}
Mansini, R., Ogryczak, W., and Speranza, M.~G. (2007).
\newblock Conditional value at risk and related linear programming models for
  portfolio optimization.
\newblock {\em Annals of Operations Research}, 152(1):227--256.

\bibitem[Nemirovski et~al., 2009]{nemirovski2009a}
Nemirovski, A., Juditsky, A., Lan, G., and Shapiro, A. (2009).
\newblock Robust stochastic approximation approach to stochastic programming.
\newblock {\em SIAM Journal on Optimization}, 19(4):1574--1609.

\bibitem[Okamoto, 1959]{okamoto1959a}
Okamoto, M. (1959).
\newblock Some inequalities relating to the partial sum of binomial
  probabilities.
\newblock {\em Annals of the Institute of Statistical Mathematics},
  10(1):29--35.

\bibitem[Prasad et~al., 2018]{prasad2018a}
Prasad, A., Suggala, A.~S., Balakrishnan, S., and Ravikumar, P. (2018).
\newblock Robust estimation via robust gradient estimation.
\newblock {\em arXiv preprint arXiv:1802.06485}.

\bibitem[Prashanth et~al., 2019]{prashanth2019a}
Prashanth, L.~A., Jagannathan, K., and Kolla, R.~K. (2019).
\newblock Concentration bounds for {CVaR} estimation: The cases of light-tailed
  and heavy-tailed distributions.
\newblock {\em arXiv preprint arXiv:1901.00997v2}.

\bibitem[Rockafellar, 1970]{rockafellar1970a}
Rockafellar, R.~T. (1970).
\newblock {\em Convex Analysis}.
\newblock Princeton University Press.

\bibitem[Rockafellar and Uryasev, 2000]{rockafellar2000a}
Rockafellar, R.~T. and Uryasev, S. (2000).
\newblock Optimization of conditional value-at-risk.
\newblock {\em Journal of Risk}, 2:21--42.

\bibitem[Shalev-Shwartz, 2012]{shalev2012a}
Shalev-Shwartz, S. (2012).
\newblock Online learning and online convex optimization.
\newblock {\em Foundations and Trends{\textregistered} in Machine Learning},
  4(2):107--194.

\bibitem[Shalev-Shwartz and Ben-David, 2014]{shalev2014a}
Shalev-Shwartz, S. and Ben-David, S. (2014).
\newblock {\em Understanding Machine Learning: From Theory to Algorithms}.
\newblock Cambridge University Press.

\bibitem[Soma and Yoshida, 2020]{soma2020a}
Soma, T. and Yoshida, Y. (2020).
\newblock Statistical learning with conditional value at risk.
\newblock {\em arXiv preprint arXiv:2002.05826}.

\bibitem[Strassen, 1965]{strassen1965a}
Strassen, V. (1965).
\newblock The existence of probability measures with given marginals.
\newblock {\em Annals of Mathematical Statistics}, 36(2):423--439.

\bibitem[Takeda and Sugiyama, 2008]{takeda2008a}
Takeda, A. and Sugiyama, M. (2008).
\newblock $\nu$-support vector machine as conditional value-at-risk
  minimization.
\newblock In {\em Proceedings of the 25th International Conference on Machine
  Learning}, pages 1056--1063.

\end{thebibliography}

\end{document}